\RequirePackage{amsmath} 

\documentclass{llncs}

\usepackage{graphicx}
\usepackage{amssymb, amsfonts}
\usepackage{xfrac}
\usepackage[sectionbib,numbers]{natbib}
\usepackage{algorithm}
\usepackage[noend]{algorithmic}
\usepackage[titlenumbered,ruled,algo2e]{algorithm2e}
\usepackage{subfigure}
\usepackage{url}
\usepackage{ulem}
\usepackage{chngcntr}
\counterwithin{figure}{section}
\usepackage{float}
\usepackage{multirow}
\restylefloat{table}
\usepackage{color}\definecolor{darkblue}{rgb}{0,0,0.75}\definecolor{darkred}{rgb}{0.51,0.02,0}
\usepackage[pdftex,pdftitle="Efficient Decentralized Deep Learning by Dynamic Model Averaging",colorlinks=true, bookmarksnumbered, citecolor=darkblue, urlcolor=darkblue, linkcolor=darkblue, pdfpagemode=UseNone]{hyperref}
\usepackage{wrapfig}

\newtheorem{pr}{Proposition}
\newtheorem{thm}[pr]{Theorem}
\newtheorem{lm}[pr]{Lemma}
\newtheorem{df}[pr]{Definition}
\newtheorem{crl}[pr]{Corollary}

\newcommand{\defemph}[1]{\textbf{#1}}

\newcommand{\innerprod}[2]{\langle #1,#2\rangle}
\newcommand{\card}[1]{\left|#1\right|}
\newcommand{\abs}[1]{\left|#1\right|}

\newcommand{\mapping}[3]{#1\!: #2 \to #3}
\newcommand{\R}{\mathbb{R}}
\newcommand{\N}{\mathbb{N}}

\newcommand{\bigO}[1]{\mathcal{O}\left( #1 \right)}

\newcommand{\loss}{\ell}

\newcommand{\cummloss}[2]{L_{#1}(#2)}
\newcommand{\lossbound}[2]{\textbf{L}_{#1}(#2)}
\newcommand{\regret}[2]{R_{#1}(#2)}
\newcommand{\regretbound}[2]{\textbf{R}_{#1}(#2)}
\newcommand{\cummcomm}[2]{C_{#1}(#2)}

\newcommand{\commcost}[1]{c_{#1}}

\newcommand{\round}{t}

\newcommand{\totalRounds}{T}
\newcommand{\learner}{i}
\newcommand{\totalLearners}{m}

\newcommand{\balancingSet}{\mathcal{B}}

\newcommand{\uprule}{\varphi}
\newcommand{\syncop}{\sigma}
\newcommand{\divergence}{\delta}
\newcommand{\divThreshold}{\Delta}

\newcommand{\minibatchSGD}{\uprule^{\text{mSGD}}}

\newcommand{\modelSpace}{\mathcal{F}}

\newcommand{\model}{f}
\newcommand{\modelconf}{\mathbf{\model}}
\newcommand{\avgmodel}{\overline{\model}}

\newcommand{\refModel}{r}

\newcommand{\onlineAlgo}{\uprule} 
\newcommand{\protocol}{\Pi}
\newcommand{\dynProt}{\mathcal{D}}
\newcommand{\periodProt}{\mathcal{P}}
\newcommand{\contProt}{\mathcal{C}}
\newcommand{\dynavg}{\syncop_{\Delta,b}}

\newcommand{\dist}{P} 

\newcommand{\sampleSpace}{X}
\newcommand{\outputSpace}{Y}
\newcommand{\sample}{x}
\newcommand{\truelabel}{y}
\newcommand{\locsample}{E}

\newcommand{\inputSpace}{\sampleSpace\times\outputSpace}

\newcommand{\repeatthanks}{\textsuperscript{\thefootnote}}

\title{Efficient Decentralized Deep Learning by Dynamic Model Averaging}
\date{}
\author{
Michael Kamp\inst{1,2,3}\thanks{These authors contributed equally.} \and
Linara Adilova\inst{1,2}\repeatthanks \and
Joachim Sicking\inst{1,2}\repeatthanks \and 
Fabian H{\"u}ger\inst{4}\and
Peter Schlicht\inst{4}\and
Tim Wirtz\inst{1,2}\and
Stefan Wrobel\inst{1,2,3}
}
\institute{Fraunhofer IAIS \email{<name>.<surname>@iais.fraunhofer.de}
\and
Fraunhofer Center for Machine Learning
\and
University of Bonn \email{<surname>@cs.uni-bonn.de}
\and  
Volkswagen Group Research \email{<name>.<surname>@volkswagen.de}
}

\begin{document}

\maketitle

\begin{abstract}
We propose an efficient protocol for decentralized training of deep neural networks from distributed data sources. 
The proposed protocol allows to handle different phases of model training equally well and to quickly adapt to concept drifts. This leads to a reduction of communication by an order of magnitude compared to periodically communicating state-of-the-art approaches. 
Moreover, we derive a communication bound that scales well with the hardness of the serialized learning problem.
The reduction in communication comes at almost no cost, as the predictive performance remains virtually unchanged. Indeed, the proposed protocol retains loss bounds of periodically averaging schemes. An extensive empirical evaluation validates major improvement of the trade-off between model performance and communication which could be beneficial for numerous decentralized learning applications, such as autonomous driving, or voice recognition and image classification on mobile phones.
\end{abstract}

\section{Introduction}
\label{sec:intro}
Traditionally, deep learning models are trained on a single system or cluster by centralizing data from distributed sources. In many applications, this requires a prohibitive amount of communication. For gradient-based training methods, communication can be reduced by calculating gradients locally and communicating the sum of gradients periodically~\citep{dekel/jmlr/2012}, instead of raw data. This mini-batch approach performs well on tightly connected distributed systems~\citep{dean2012large, zhang2015deep, chen2016revisiting} (e.g., data centers and clusters). For many applications, however, centralization or even periodic sharing of gradients between local devices becomes infeasible due to the large amount of necessary communication.

For decentralized systems with limited communication infrastructure it was suggested to compute local updates~\citep{zinkevich/nips/2010} and average models periodically, instead of sharing gradients. 
Averaging models 
has three major advantages: (i) sending only the model parameters instead of a set of data samples reduces communication\footnote{Note that averaging models requires the same amount of communication as sharing gradients, since the vector of model parameters is of the same dimension as the gradient vector of the loss function.}; (ii) it allows to train a joint model without exchanging or centralizing privacy-sensitive data; and (iii) it can be applied to a wide range of learning algorithms, since it treats the underlying algorithm as a black-box.

This approach is used in convex optimization~\citep{shamir2016without, mcdonald2009efficient, zhang2012communication}. For non-convex objectives, a particular problem is that the average of a set of models can have a worse performance than any model in the set---see Figure~\ref{fig:illustration}\subref{fig:illustration:nonconv}. 
For the particular case of deep learning,~\citet{mcmahan2017communication} empirically evaluated model averaging in decentralized systems and termed it \defemph{Federated Learning}.

However, averaging periodically still invests communication independent of its utility, e.g., when all models already converged to an optimum. This disadvantage is even more apparent in case of concept drifts: periodic approaches cannot react adequately to drifts, since they either communicate so rarely that the 
models adapt too slowly to the change, or so frequently that they generate an immense amount of unnecessary communication in-between drifts.
\begin{figure}[t]
	\centering
	\subfigure[]{\includegraphics[width=6cm]{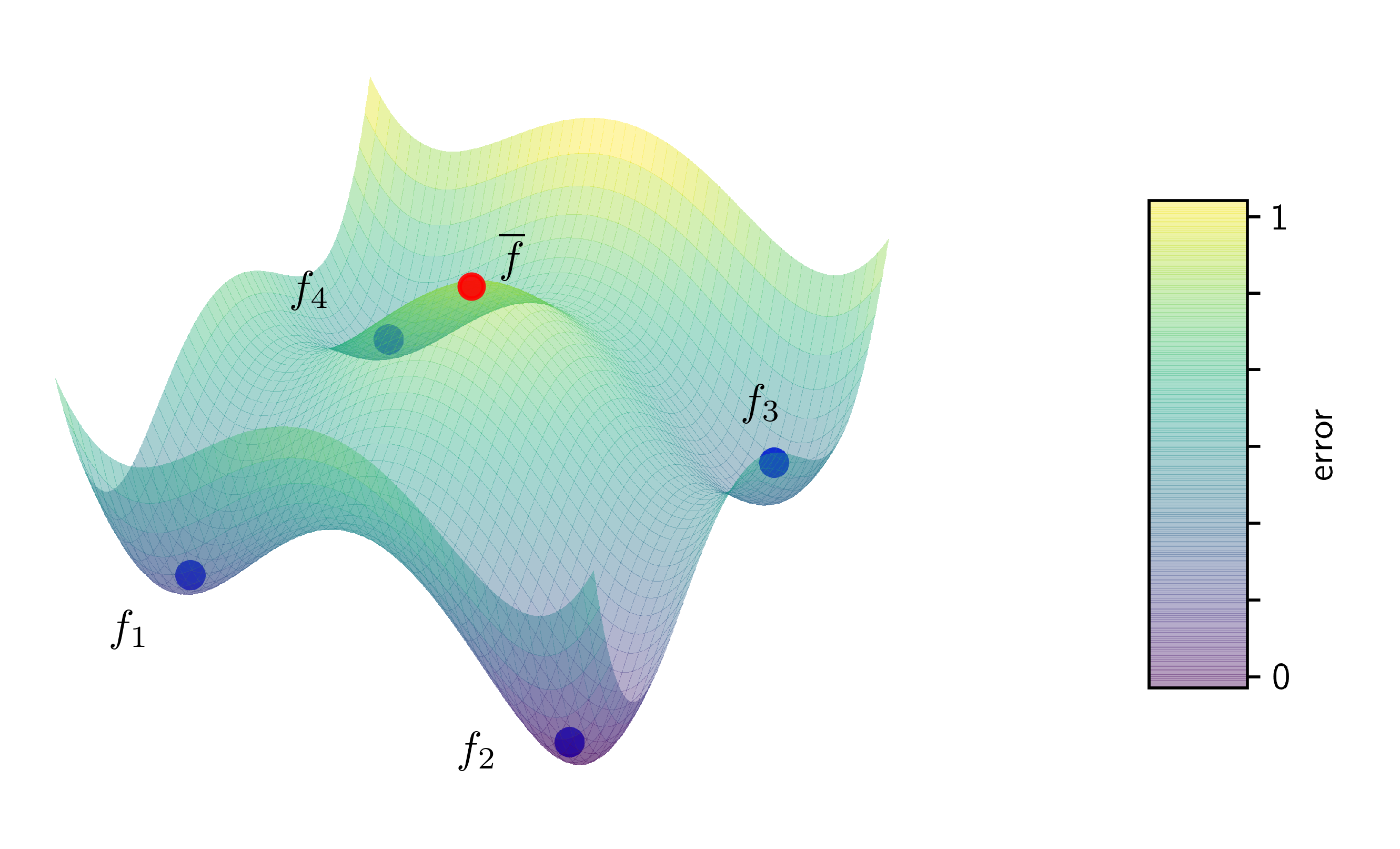}\label{fig:illustration:nonconv}}
	\hfill
	\subfigure[]{\includegraphics[width=6cm]{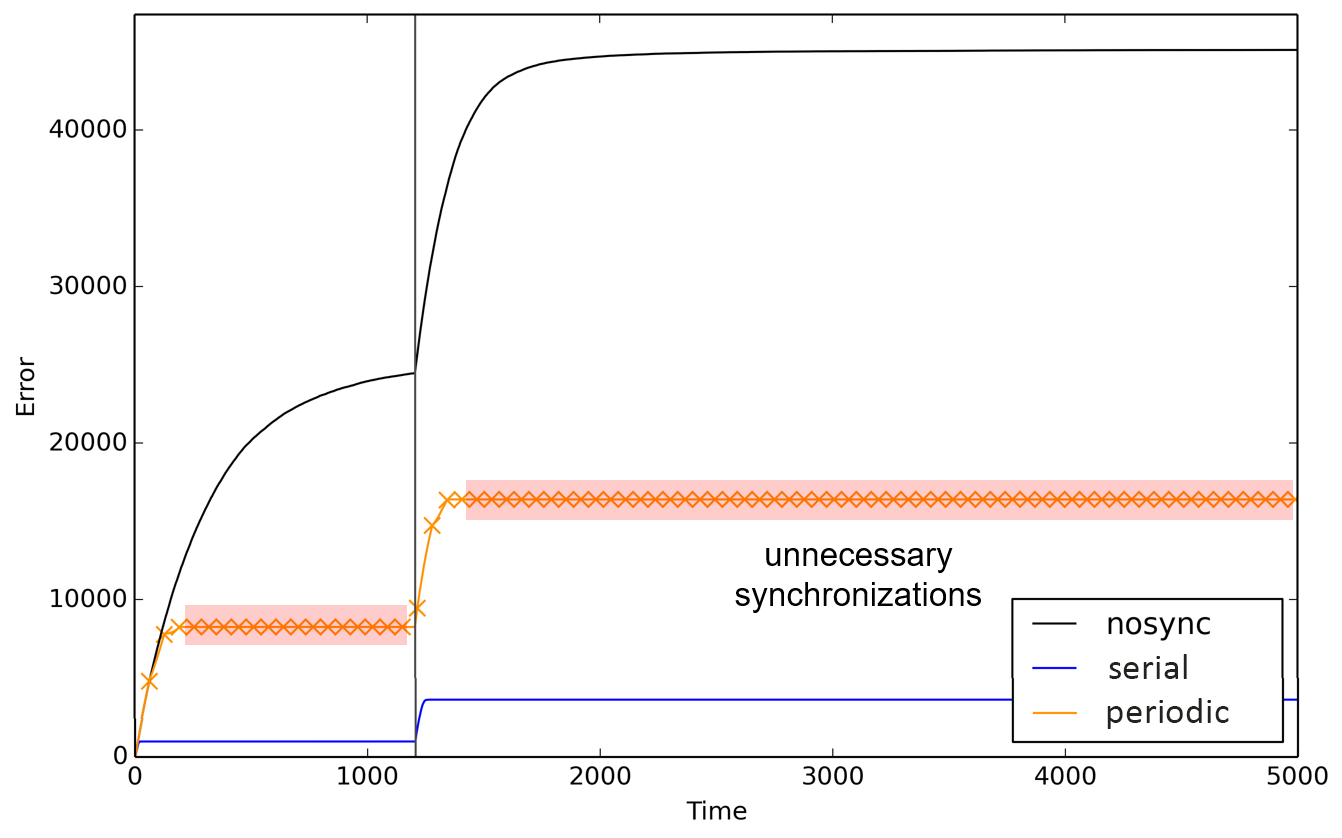}\label{fig:illustration:unnecessaryComm}}	
	\caption{ \subref{fig:illustration:nonconv} Illustration of the problem of averaging models in non-convex problems: each of the models $f_1,\dots,f_4$ has reached a local minimum, but their average $\overline{f}$ has a larger error than each of them.
	\subref{fig:illustration:unnecessaryComm} Cumulative error over time for a serial learning algorithm and two decentralized learning algorithms with $10$ learners, one that does not communicate (nosync) and one that communicates every $50$ time steps (periodic). The vertical line indicates a concept drift, i.e., a rapid change in the target distribution.}
	\label{fig:illustration}
\end{figure}

In~\citet{kamp2014communication} the authors proposed to average models dynamically, depending on the utility of the communication. 
The main idea is to reduce communication without losing predictive performance by investing the communication efficiently: When local learners do not suffer loss, communication is unnecessary and should be avoided (see Figure~\ref{fig:illustration}\subref{fig:illustration:unnecessaryComm}); similarly, when they suffer large losses, an increased amount of communication should be invested to improve their performances. The problem setting and a criterion for efficient approaches is defined in Section~\ref{sec:preliminaries}. 
This approach, denoted \defemph{dynamic averaging}, was proposed for online learning convex objectives~\citep{kamp2014communication,kamp2016kernels}.
We adapt dynamic averaging to the non-convex objectives of deep learning in Section~\ref{sec:dynProtocol}. 

Our contribution is the description and evaluation of a general method for decentralized training of deep neural networks that (i) substantially reduces communication while retaining high predictive performance and (ii) is in addition well-suited to concept drifts in the data. To that end, Section~\ref{sec:theory} shows that, for common learning algorithms, dynamic averaging is an efficient approach for non-convex problems, i.e., it retains the predictive performance of a centralized learner but is also adaptive to the current hardness of the learning problem.

A natural application for dynamic decentralized machine learning is \defemph{in-fleet learning} of autonomous driving functionalities: concept drifts occur naturally, since properties central for the modeling task may change---changing traffic behavior both over time and different countries or regions introduce constant and unforeseeable concept drifts. Moreover, large high-frequency data streams generated by multiple sensors per vehicle renders data centralization prohibitive in large fleets. 
Section~\ref{sec:exps} provides an extensive empirical evaluation of the dynamic averaging approach on classical deep learning tasks, as well as synthetic and real-world tasks with concept drift, including in-fleet learning of autonomous driving functionalities. The approach is compared to periodically communicating schemes, including \defemph{Federated Averaging}~\citep{mcmahan2017communication}, a state-of-the-art approach for decentralized deep learning---more recent approaches are interesting from a theoretical perspective but show no practical improvement~\citep{jiang2017collaborative}, or tackle other aspects of federated learning, such as non-iid data~\citep{smith2017federated} or privacy aspects~\citep{mcmahan2018Learning}.

Section~\ref{sec:discussion} discusses properties and limitations of dynamic averaging and puts it into context of related work, followed by a conclusion in Section~\ref{sec:conclusion}.

\section{Preliminaries}
\label{sec:preliminaries}
We consider a decentralized learning setting with $\totalLearners\in\N$ \defemph{local learners}, where each learner $\learner\in [\totalLearners]$ runs the same \defemph{learning algorithm} $\mapping{\uprule}{\modelSpace\times 2^X\times 2^Y}{\modelSpace}$ that trains a \defemph{local model} $\model^\learner$ from a \defemph{model space} $\modelSpace$ using local samples from an \defemph{input space} $\sampleSpace$ and \defemph{output space} $\outputSpace$.
We assume a streaming setting, where in each round $\round\in\N$ each learner $\learner\in [\totalLearners]$ observes a sample $\locsample_\round^\learner\subset\sampleSpace\times\outputSpace$ of size $|\locsample_\round^\learner|=B$, drawn iid from the same time variant distribution $\mapping{\dist_\round}{\inputSpace}{\R_+}$. 
The local learner uses its local model to make a prediction whose quality is measured by a \defemph{loss function} $\mapping{\loss}{\modelSpace\times\inputSpace}{\R_+}$. We abbreviate the loss of the local model of learner $\learner$ in round $\round$ by
$
\loss_\round^\learner\left(\model_{\round}^\learner\right) = \sum_{\left(\sample,\truelabel\right)\in\locsample_\round^\learner}\loss\left(\model_\round^\learner,\sample,\truelabel \right)
$\footnote{This setup includes online learning ($B=1$) and mini-batch training $B>1$. The gradient of $\loss_\round^\learner$ is the sum of individual gradients. Our approach and analysis also apply to heterogeneous sampling rates $B^\learner$ for each learner $\learner$.}.
%
The goal of decentralized learning is to minimize the \defemph{cumulative loss} up to a time horizon $\totalRounds\in\N$, i.e., 
\begin{equation}
\cummloss{}{\totalRounds,\totalLearners}=\sum_{\round=1}^{\totalRounds}\sum_{\learner=1}^{\totalLearners}\loss^\learner_\round\left(\model^\learner_\round\right)\enspace .
\label{eq:cumulativeLoss}
\end{equation}
Guarantees on the predictive performance, measured by the cumulative loss, are typically given by a \defemph{loss bound} $\lossbound{}{\totalRounds,\totalLearners}$. That is, for all possible sequences of losses it holds that $\cummloss{}{\totalRounds,\totalLearners}\leq \lossbound{}{\totalRounds,\totalLearners}$. 

In each round $\round\in\N$, local learners use a \defemph{synchronization operator} $\mapping{\syncop}{\modelSpace^\totalLearners}{\modelSpace^\totalLearners}$ that transfers the current set of local models, called the current \defemph{model configuration} ${\modelconf_\round=\{\model_\round^1,\dots,\model_\round^\totalLearners\}}$, into a single stronger \defemph{global model} $\syncop(\modelconf_\round)$ which replaces the local models.
We measure the performance of the operator in terms of communication by the \defemph{cumulative communication}, i.e.,
\[
\cummcomm{}{\totalRounds,\totalLearners}=\sum_{\round=1}^{\totalRounds}\commcost{}(\modelconf_{\round})\enspace ,
\]
where $\mapping{\commcost{}}{\modelSpace^\totalLearners}{\N}$ measures the number of bytes required by the protocol to synchronize the models $\modelconf_\round$
 at time $\round$.
We investigate synchronization operators that aggregate models by computing their average~\citep{mcmahan2017communication, mcdonald2009efficient, shamir2016without, zinkevich/nips/2010, zhang2012communication}, i.e., 
$
{\avgmodel = \sfrac{1}{\totalLearners}\sum_{\learner= 1}^{\totalLearners}\model^\learner}
$.
In the case of neural networks, we assume that all local models have the same architecture, thus their average is the average of their respective weights. 
We discuss the potential use of other aggregation operations in Section~\ref{sec:discussion}.
We denote the choice of learning algorithm together with the synchronization operator as a \defemph{decentralized learning protocol} $\protocol=(\uprule,\syncop)$. The protocol is evaluated in terms of the predictive performance and cumulative communication. 
In order to assess the efficiency of decentralized learning protocols in terms of the trade-off between loss and communication,~\citet{kampAdaptive2014} introduced two criteria: consistency and adaptiveness. 
\begin{df}[\citet{kampAdaptive2014}]
	A distributed online learning protocol $\protocol=(\uprule,\syncop)$ processing $\totalLearners\totalRounds$ inputs is \defemph{consistent} if it retains the loss of the serial online learning algorithm $\uprule$, i.e., 
	\[
	\cummloss{\protocol}{\totalRounds,\totalLearners} \in \bigO{ \cummloss{\uprule}{\totalLearners\totalRounds} }\enspace .
	\]
	The protocol is \defemph{adaptive} if its communication bound is linear in the number of local learners $\totalLearners$ and the loss $\cummloss{\uprule}{\totalLearners\totalRounds}$ of the serial online learning algorithm, i.e.,
	\[
	\cummcomm{\protocol}{\totalRounds,\totalLearners} \in \bigO{ \totalLearners \cummloss{\uprule}{\totalLearners\totalRounds} }\enspace .
	\]
	\label{def:efficiency}
\end{df}
A decentralized learning protocol is \defemph{efficient} if it is both consistent and adaptive.
Each one of the criteria can be trivially achieved: A non-synchronizing protocol is adaptive but not consistent, a protocol that centralizes all data is consistent but not adaptive. Protocols that communicate periodically are consistent~\citep{dekel/jmlr/2012, zinkevich/nips/2010}, i.e., they achieve a predictive performance comparable to a model that is learned centrally on all the data. However, they require an amount of communication linear in the number of learners $\totalLearners$ and the number of rounds $\totalRounds$, independent of the loss. Thus they are not adaptive.

In the following section, we recapitulate dynamic averaging and apply it to the non-convex problem of training deep neural networks. In Section~\ref{sec:theory} we discuss in which settings it is efficient as in Definition~\ref{def:efficiency}.


\section{Dynamic Averaging}
\label{sec:dynProtocol}

In this section, we recapitulate the dynamic averaging protocol~\citep{kamp2016kernels} for synchronizations based on quantifying their effect (Algorithm~\ref{alg:protocol}).
%
Intuitively, communication is not well-invested in situations where all models are already approximately equal---either because they were updated to similar models or have merely changed at all since the last averaging step---and it is more effective if models are diverse. A simple measure to quantify the effect of synchronizations is given by the \defemph{divergence} of the current model configuration, i.e., 
\begin{equation}
\divergence({\modelconf})=\frac{1}{\totalLearners}\sum_{\learner=1}^\totalLearners \left\|\model^\learner-\avgmodel\right\|^2 \enspace .
\label{eq:defDivergence}
\end{equation}
Using this, we define the dynamic averaging operator
that allows to omit synchronization in cases where the divergence of a model configuration is low.
\begin{df}[\citet{kamp2014communication}]
	\label{def:dynamicsync}
	A \defemph{dynamic averaging operator} with positive divergence threshold $\Delta \in \R_+$ and batch size $b \in \N$ is a synchronization operator $\dynavg$ such that $\dynavg(\modelconf_t)=\modelconf_t$ if $\round \mod b \neq 0$ and otherwise: (i) $\avgmodel_t=\overline{\dynavg(\modelconf_t)}$, i.e., it leaves the mean model invariant, and (ii) $\divergence\left(\dynavg(\modelconf_\round)\right) \leq \Delta$, i.e., after its application the model divergence is bounded by $\Delta$.
\end{df}
An operator adhering to this definition does not generally put all nodes into sync (albeit we still refer to it as \textit{synchronization} operator).
In particular it allows to leave all models untouched as long as the divergence remains below $\Delta$ or to only average a subset of models in order to satisfy the divergence constraint. 

\begin{algorithm2e}[ht]
	\caption{Dynamic Averaging Protocol}
	\label{alg:protocol}
    \smallskip
    \textbf{Input:}    divergence threshold $\Delta$, batch size $b$\\    
	\smallskip
	\textbf{Initialization:}\\
	\begin{algorithmic}[0]
		\STATE local models $\model^1_1,\dots,\model^\totalLearners_1 \leftarrow$ one random $\model$
		\STATE reference vector $r \leftarrow \model$
		\STATE violation counter $v \leftarrow 0$
	\end{algorithmic}
	\smallskip
	\textbf{Round }$\round$\textbf{ at node }$\learner$\textbf{:}\\
	\begin{algorithmic}[0]
		\STATE \textbf{observe} $\locsample_\round^\learner\subset\sampleSpace\times\outputSpace$ 
		\STATE \textbf{update} $\model^\learner_{\round-1}$ using the learning algorithm $\uprule$\\
		\IF{$\round \mod b=0$ \textbf{and} $\|\model^\learner_{\round}-r\|^2> \Delta$}
		\STATE \textbf{send} $\model^\learner_{\round}$ to coordinator (violation)
		\ENDIF
	\end{algorithmic}
	\smallskip
	\textbf{At coordinator on violation:}\\
	\begin{algorithmic}[0]
		\STATE \textbf{let} $\balancingSet$ be the set of nodes with violation
		\STATE $v\leftarrow v+\card{\balancingSet}$
		\STATE \textbf{if} $v=\totalLearners$ \textbf{then} $\balancingSet\leftarrow [\totalLearners]$, $v\leftarrow 0$
		\WHILE{$\balancingSet \neq [\totalLearners]$ \textbf{and} $\left\|\frac{1}{\balancingSet}\sum_{\learner \in \balancingSet}\model^\learner_\round-r\right\|^2> \Delta$}
		\STATE \textbf{augment} $\balancingSet$ by augmentation strategy
		\STATE \textbf{receive} models from nodes added to $\balancingSet$
		\ENDWHILE
		\STATE \textbf{send} model $\avgmodel=\frac{1}{\balancingSet}\sum_{\learner \in \balancingSet}\model^\learner_\round$ to nodes in $\balancingSet$
		\STATE \textbf{if} $\balancingSet=[\totalLearners]$ also set new reference vector $r\leftarrow\avgmodel$
	\end{algorithmic}
\end{algorithm2e}
The \defemph{dynamic averaging protocol} $\dynProt=(\uprule,\dynavg)$ synchronizes the local learners using the dynamic averaging operator $\dynavg$.
This operator only communicates when the model divergence 
exceeds a \defemph{divergence threshold} $\divThreshold$. 
In order to decide when to communicate locally, at round $\round\in\N$, each local learner $\learner\in [\totalLearners]$  monitors the \defemph{local condition} $\|\model_{\round}^{\learner}-\refModel\|^2\leq\Delta$ for a \defemph{reference model} $\refModel\in\modelSpace$~\citep{sharfman2008shape} that is common among all learners (see~\citep{keren2012shape,sharfman/tods/2007,gabel/IPDPS/2014,lazerson2015monitoring,keren2014geometric} for a more general description of this method). The local conditions guarantee that if none of them is violated, i.e., for all $\learner\in [\totalLearners]$ it holds that $\|\model^\learner_\round - r\|^2\leq\Delta$, then the divergence does not exceed the threshold, i.e., $\divergence(\modelconf_\round)\leq\Delta$~\citep[Theorem~6]{kamp2014communication}.
The closer the reference model is to the true average of local models, the tighter are the local conditions. 
Thus, the first choice for the reference model is the average model from the last synchronization step. The local condition is checked every $b\in\N$ rounds. This allows using the common mini-batch approach~\citep{bottou1991stochastic} for training deep neural networks.

If one or more local conditions are violated, all local models can be averaged---an operation referred to as \defemph{full synchronization}. However, on a local violation the divergence threshold is not necessarily crossed. In that case, the violations may be locally balanced: the coordinator incrementally queries other local learners for their models; if the average of all received models lies within the safe zone, it is transferred back as new model to all participating nodes.
If all nodes have been queried, the result is equivalent to a full synchronization and the reference vector is updated.
In both cases, the divergence of the model configuration is bounded by $\Delta$ at the end of the balancing process, because all local conditions hold.
Also, it is easy to check that this protocol leaves the global mean model unchanged. Hence, it is complying to Def.~\ref{def:dynamicsync}. 
In the following Section, we theoretically analyze the loss and communication of dynamic averaging.

\section{Efficiency of Dynamic Averaging}
\label{sec:theory}
In order to assess the predictive performance and communication cost of the dynamic averaging protocol for deep learning, we compare it to a periodically averaging approach: Given a learning algorithm $\onlineAlgo$, the \defemph{periodic averaging protocol} $\periodProt=(\onlineAlgo,\syncop_b)$ synchronizes the current model configuration ${\modelconf}$ every $b\in\N$ time steps by replacing all local models by their joint average
$
\avgmodel=\sfrac{1}{\totalLearners}\sum_{\learner=1}^{\totalLearners}\model^\learner
$. That is, the synchronization operator is given by
\[
\syncop_b(\modelconf_t)=	\begin{cases}
\left(\avgmodel_\round,\dots,\avgmodel_\round\right), &\text{ if } b \equiv O(\round)\\
\enspace\modelconf_\round = (\model_\round^1,\dots,\model_\round^\totalLearners), &\text{ otherwise}\\
\end{cases}\enspace .
\]
A special case of this is the \defemph{continuous averaging protocol} $\contProt=(\onlineAlgo,\syncop_1)$, synchronizing every round, i.e., for all $\round\in\N$, the synchronization operator is given by $
\syncop_1\left(\modelconf_\round\right)=\left(\avgmodel_\round,\dots,\avgmodel_\round\right)
$.
As base learning algorithm we use mini-batch SGD algorithm $\minibatchSGD_{B,\eta}$~\citep{dekel/jmlr/2012} with mini-batch size $B\in\N$ and learning rate $\eta\in\R_+$ commonly used in deep learning~\citep{bottou1991stochastic}. One step of this learning algorithm given the model $\model\in\modelSpace$ can be expressed as
\[
\minibatchSGD_{B,\eta}(\model) = \model -\eta\sum_{j=1}^B \nabla\loss^j(\model)\enspace .
\]
Let $\contProt^{\text{mSGD}}=(\minibatchSGD_{B,\eta},\syncop_1)$ denote the continuous averaging protocol using mini-batch SGD. For $\totalLearners\in\N$ learners with the same model $\model\in\modelSpace$, $\totalLearners B$ training samples 
$(\sample_1,\truelabel_1),\dots,(\sample_{\totalLearners B},\truelabel_{\totalLearners B})$, and
corresponding loss functions $\loss^\learner(\cdot) = \loss(\cdot,\sample_\learner,\truelabel_\learner)$, one step of $\contProt^{\text{mSGD}}$ is 
\[
\syncop_1\left(\left(\minibatchSGD_{B,\eta}(\model),\dots,\minibatchSGD_{B,\eta}(\model)\right)\right) = \frac{1}{\totalLearners}\sum_{\learner=1}^{\totalLearners}\left(\model - \eta\textstyle\sum_{j=1}^{B}\nabla\loss^{(\learner - 1)B + j}(\model)\right)\enspace .
\]
We compare $\contProt^{\text{mSGD}}$ to the serial application of mini-batch SGD. 
%
It can be observed that continuous averaging with mini-batch SGD on $\totalLearners\in\N$ learners with mini-batch size $B$ is equivalent to serial mini-batch SGD with a mini-batch size of $\totalLearners B$ and a learning rate that is $\totalLearners$ times smaller.
\begin{pr}
For $\totalLearners\in\N$ learners, a mini-batch size $B\in\N$, $\totalLearners B$ training samples $(\sample_1,\truelabel_1),\dots,(\sample_{\totalLearners B},\truelabel_{\totalLearners B})$, corresponding loss functions $\loss^\learner(\cdot) = \loss(\cdot,\sample_\learner,\truelabel_\learner)$, a learning rate $\eta\in\R_+$, and a model $\model\in\modelSpace$, it holds that
\[
\syncop_1\left(\left(\minibatchSGD_{B,\eta}(\model),\dots,\minibatchSGD_{B,\eta}(\model)\right)\right) = \minibatchSGD_{\totalLearners B,\sfrac{\eta}{\totalLearners}}(\model)\enspace .
\]
\label{prop:contAvgEqualMiniBatch}
\end{pr}
\begin{proof}
\begin{equation*}
\begin{split}
\syncop_1&\left(\left(\minibatchSGD_{B,\eta}(\model),\dots,\minibatchSGD_{B,\eta}(\model)\right)\right) = \frac{1}{\totalLearners}\sum_{\learner=1}^{\totalLearners}\left(\model - \eta\sum_{j=1}^{B}\nabla\loss^{(\learner - 1)B + j}(\model)\right)\\
&= \frac{1}{\totalLearners}\totalLearners \model - \frac{1}{m}\eta\sum_{\learner=1}^{\totalLearners}\sum_{j=1}^{B}\nabla\loss^{(\learner - 1)B + j}(\model)=\model  - \frac{1}{m}\eta\sum_{j=1}^{\totalLearners B}\nabla\loss^{j}(\model)=\minibatchSGD_{\totalLearners B,\sfrac{\eta}{\totalLearners}}(\model)
\end{split}
\end{equation*}
\qed
\end{proof}
In particular, Proposition~\ref{prop:contAvgEqualMiniBatch} holds for continuous averaging with a mini-batch size of $B=1$, i.e., classic stochastic gradient descent. 
From Proposition~\ref{prop:contAvgEqualMiniBatch} it follows that continuous averaging is consistent as in Definition~\ref{def:efficiency}, since it retains the loss bound of serial mini-batch SGD and classic SGD. If the loss function is locally convex in an $\bigO{\Delta}$-radius around the current average---a non-trivial but realistic assumption~\citep{nguyen2017loss,keskar2017on}---Theorem~2 in~\citet{boley2013communication} 
guarantees that for SGD, dynamic averaging has a predictive performance similar to any periodically communicating protocol, in particular to $\syncop_1$ (see Appendix~\ref{app:sec:theory} for details).
For this case it follows that dynamic averaging using SGD for training deep neural networks is consistent.
Theorem~2 in~\citet{kampAdaptive2014} shows that the cumulative communication of the dynamic averaging protocol using SGD and a divergence threshold $\Delta$ is bounded by
\[
\cummcomm{}{\totalRounds,\totalLearners}\in\bigO{\frac{\commcost{}(\modelconf)}{\sqrt{\Delta}}\cummloss{}{\totalRounds,\totalLearners}}\enspace ,
\]
where $\commcost{}(\modelconf)$ is the number of bytes required to be communicated to average a set of deep neural networks. Since each neural network has a fixed number of weights, $\commcost{}(\modelconf)$ is in $\bigO{\totalLearners}$. It follows that dynamic averaging is adaptive.
Thus, using dynamic averaging with stochastic gradient descent for the decentralized training of deep neural networks is efficient as in Definition~\ref{def:efficiency}.

Note that the synchronization operator can be implemented using different assumptions on the system's topology and communication protocol, i.e., in a peer-to-peer fashion, or in a hierarchical communication scheme. For simplicity, in our analysis of the communication of different synchronization operators we assume that the synchronization operation is performed by a dedicated coordinator node. This coordinator is able to poll local models, aggregate them and send the global model to the local learners. 

\section{Empirical Evaluation}
\label{sec:exps}
This section empirically evaluates dynamic averaging for training deep neural networks.
To emphasize the theoretical result from Section~\ref{sec:theory}, we show that dynamic averaging indeed retains the performance of periodic averaging with substantially less communication. This is followed by a comparison of our approach with a state-of-the-art communication approach\footnote{The code of the experiments is available at \url{https://bitbucket.org/Michael_Kamp/decentralized-machine-learning}.}.
%
%
%
The performance is then evaluated in the presence of concept drifts. 
Combining the aforementioned aspects, we apply our protocol to a non-convex objective with possible concept drifts from the field of autonomous driving. 

Throughout this section, if not specified separately, we consider 
mini-batch SGD $\minibatchSGD_{B,\eta}$
as learning algorithm, since recent studies indicate that it is particularly suited for training deep neural networks~\citep{zhang2017understanding}. That is, we consider communication protocols $\Pi=(\minibatchSGD_{B,\eta}, \sigma)$ with various synchronization operators $\syncop$. The hyper-parameters of the protocols and the mini-batch SGD have been optimized on an independent dataset. Details on the experiments, including network architectures, can be found in the Appendix~\ref{app:sec:exps}.

\subsubsection{Dynamic Averaging for Training Deep Neural Networks:}\label{ssec:non_convex}
%
%
To evaluate the performance of dynamic averaging in deep learning, we first compare it to periodic averaging for training a convolutional neural network (CNN) on the MNIST classification dataset~\citep{lecun1998mnist}. We furthermore compare both protocols to a non-synchronizing protocol, denoted \defemph{nosync}, and a serial application of the learning algorithm on all data, denoted \defemph{serial}. 
%



%
\begin{wrapfigure}{r}{0.5\textwidth}
	\includegraphics[width=0.5\textwidth]{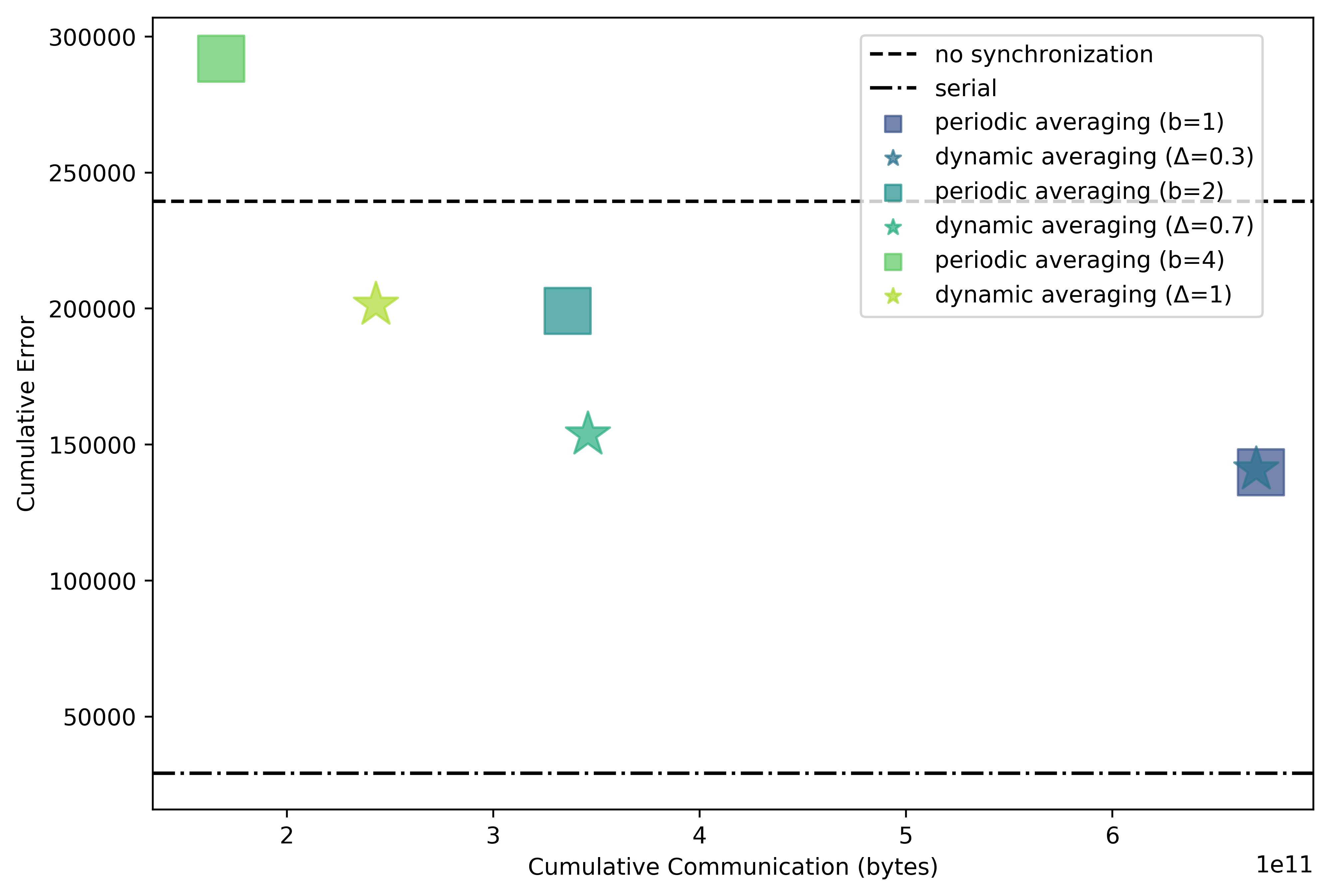}
	\caption{Cumulative loss and communication of distributed learning protocols with $m=100$ (similar to~\citet{mcmahan2017communication}) learners with mini-batch size $B=10$, each observing $T=14000$ samples (corresponding to $20$ epochs for the serial baseline). 
	}
	\label{fig:100nodes_per14000_performance_main}
\end{wrapfigure} 
Figure~\ref{fig:100nodes_per14000_performance_main} shows the cumulative error of several setups of dynamic and periodic averaging, as well as the nosync and serial baselines. The experiment confirms that for each setup of the periodic averaging protocol a setup of dynamic averaging can be found that reaches a similar predictive performance with substantially less communication (e.g., a dynamic protocol with $\sigma_{\Delta=0.7}$ reaches a performance comparable to a periodic protocol with $\sigma_{b=1}$ using only half of the communication). 
The more learners communicate, the lower their cumulative loss, with the serial baseline performing the best. 

The advantage of the dynamic protocols over the periodic ones in terms of communication is in accordance with the convex case. For large synchronization periods, however, synchronizing protocols ($\syncop_{b=4}$) have even larger cumulative loss than the nosync baseline. This behavior cannot happen in the convex case, where averaging is always superior to not synchronizing~\citep{kamp2014communication}. In contrast, in the non-convex case local models can converge to different local minima. Then their average might have a higher loss value than each one of the local models (as illustrated in Figure~\ref{fig:illustration}\subref{fig:illustration:nonconv}). 
\subsubsection{Comparison of the Dynamic Averaging Protocol with FedAvg:}\label{ssec:fedsgd}
\begin{figure}[t]
	\centering
	\begin{minipage}{5.8cm}
		\includegraphics[width=5.4cm]{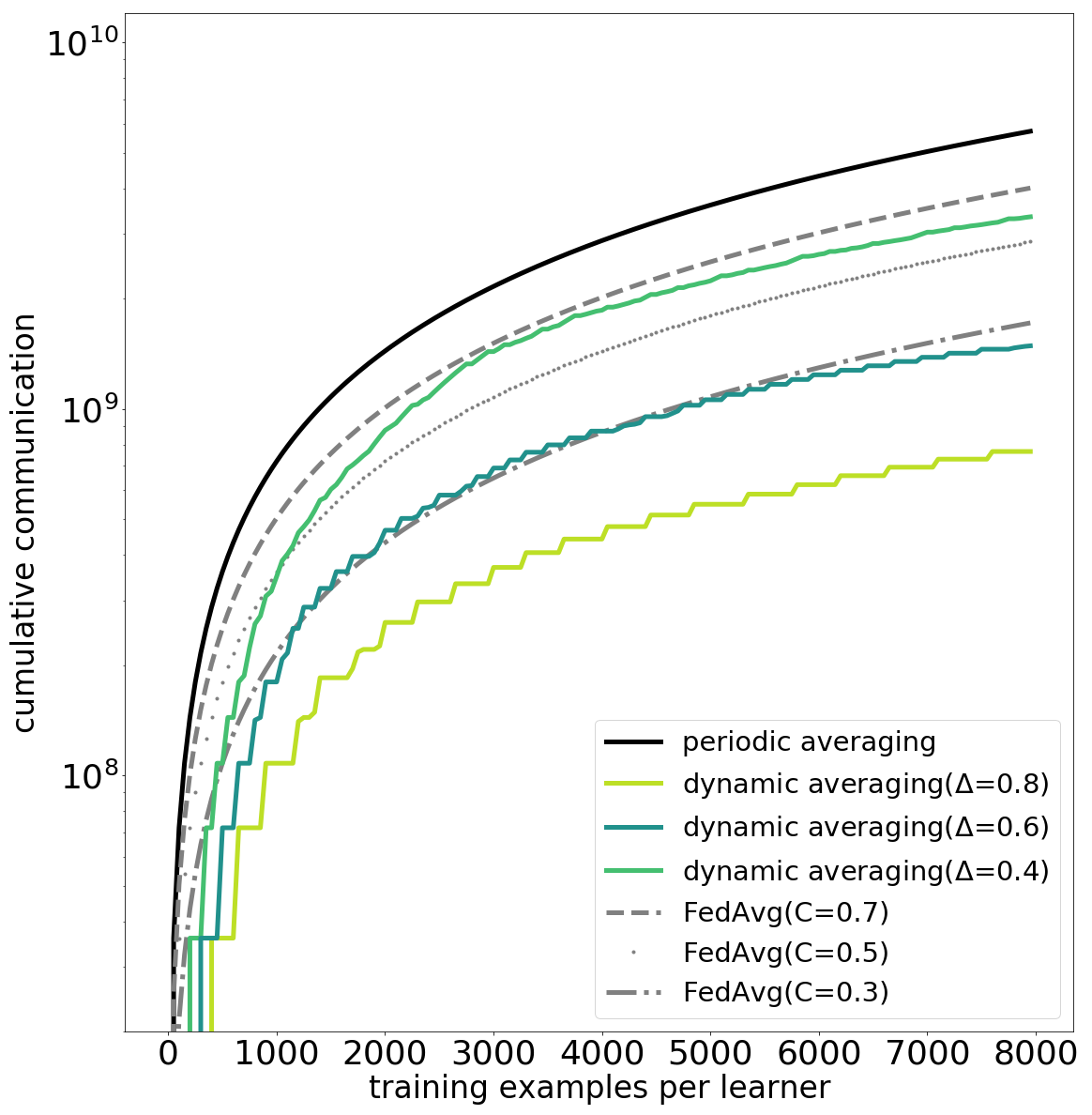}
		\caption{Evolution of cumulative communication for different dynamic averaging and FedAvg protocols on $m=30$ learners using a mini-batch size $B=10$.	\vspace{-0.3cm}}
		\label{fig:fedsgd_errorCom_1}
	\end{minipage}
	\hfill
	\begin{minipage}{5.8cm}
		\includegraphics[width=5.4cm]{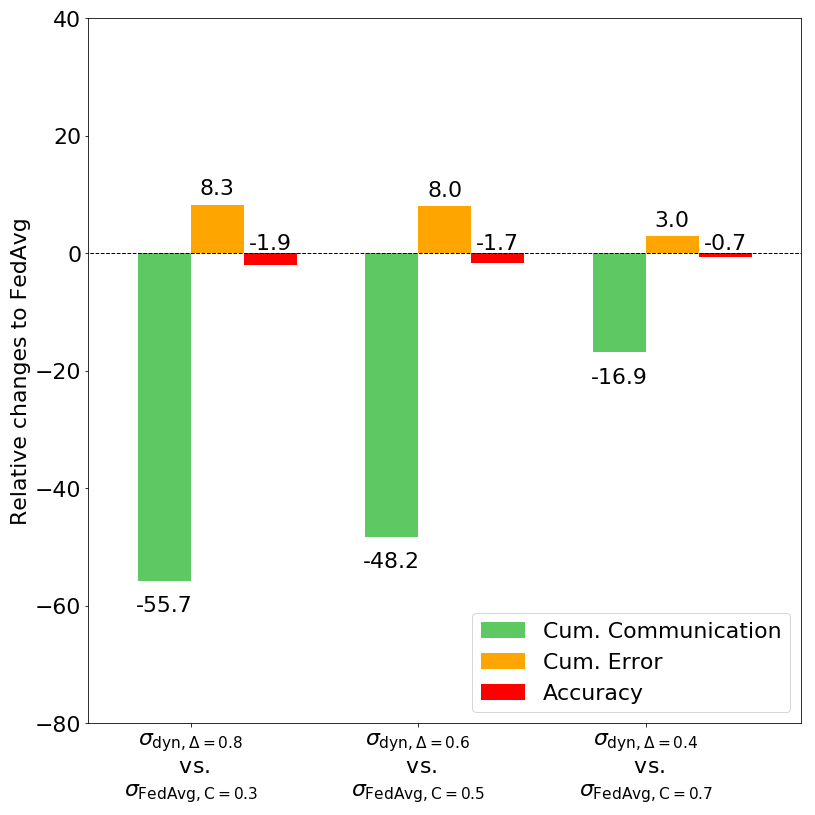}
		\caption{Comparison of the best performing settings of the dynamic averaging protocol with their \mbox{FedAvg} counterparts.\vspace{-0.3cm}}
		\label{fig:fedsgd_errorCom2_2}
	\end{minipage}
\end{figure}
Having shown that dynamic averaging outperforms standard periodic averaging, we proceed by comparing it to a highly communication-efficient variant of periodic averaging, denoted \defemph{FedAvg}~\citep{mcmahan2017communication}, which poses a state-of-the-art for decentralized deep learning under communication-cost constraints.
%

Using our terminology, FedAvg is a periodic averaging protocol that uses only a randomly sampled subset of nodes in each communication round. This subsampling leads to a reduction of total communication by a constant factor compared to standard periodic averaging. In order to compare dynamic averaging to FedAvg, we repeat the MNIST classification using CNNs and multiple configurations of dynamic averaging and FedAvg. 

Figure~\ref{fig:fedsgd_errorCom_1} shows the evolution of cumulative communication during model training comparing dynamic averaging to the optimal configuration of FedAvg with $b=5$ and $C=0.3$ for MNIST (see Section~3 in~\citet{mcmahan2017communication}) and variants of this configuration. 
We find noteworthy spreads between the communication curves, while all approaches have comparable losses. The communication amounts of all FedAvg variants increase linearly during training. 
The smaller the fraction of learners, $C\in (0,1]$, involved in synchronization, the smaller the amount of communication. In contrast, we observe step-wise increasing curves for all dynamic averaging protocols which reflect their inherent irregularity of communication.
Dynamic averaging with  $\Delta=0.6$ and $\Delta=0.8$ beat the strongest FedAvg configuration in terms of cumulative communication, the one with $\Delta=0.8$ even with a remarkable margin. We find these improvements of communication efficiency to come at almost no cost: Figure~\ref{fig:fedsgd_errorCom2_2} compares the three strongest configurations of dynamic averaging to the best performing FedAvg ones, showing a reduction of over $50\%$ in communication with an increase in cumulative loss by only $8.3\%$. The difference in terms of classification accuracy is even smaller, dynamic averaging is only worse by $1.9\%$. Allowing for more communication improves the loss of dynamic averaging to the point where dynamic averaging has virtually the same accuracy as FedAvg with $16.9\%$ less communication.

\vspace{-0.2cm}
\subsubsection{Adaptivity to Concept Drift:}\label{ssec:drift}
The advantage of dynamic averaging over any periodically communicating protocol lies in the adaptivity to the current hardness of the learning problem, measured by the in-place loss. For fixed target distributions, this loss decreases over time so that the dynamic protocol reduces the amount of communication continuously until it reaches quiescence, if no loss is suffered anymore. In the presence of concept drifts, such quiescence can never be reached; after each drift, the learners have to adapt to the new target. 
\begin{figure}[t]
	\centering
	\subfigure[cumulative loss and communication]{\label{fig:conceptDriftPerf}\includegraphics[height=3.9cm]{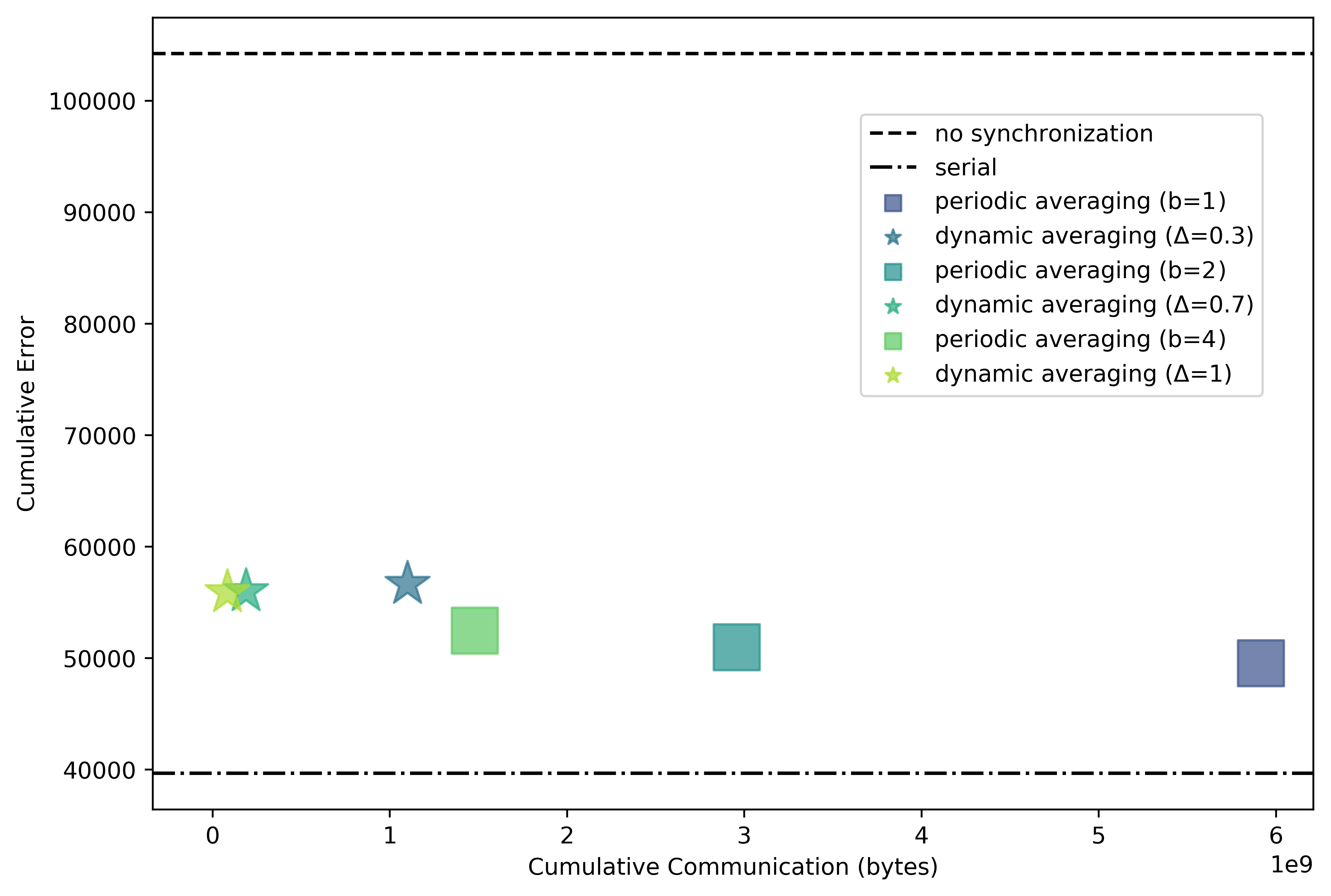}}
	\hfill
	\subfigure[cumulative communication]{\label{fig:conceptDriftComm}\includegraphics[height=3.9cm]{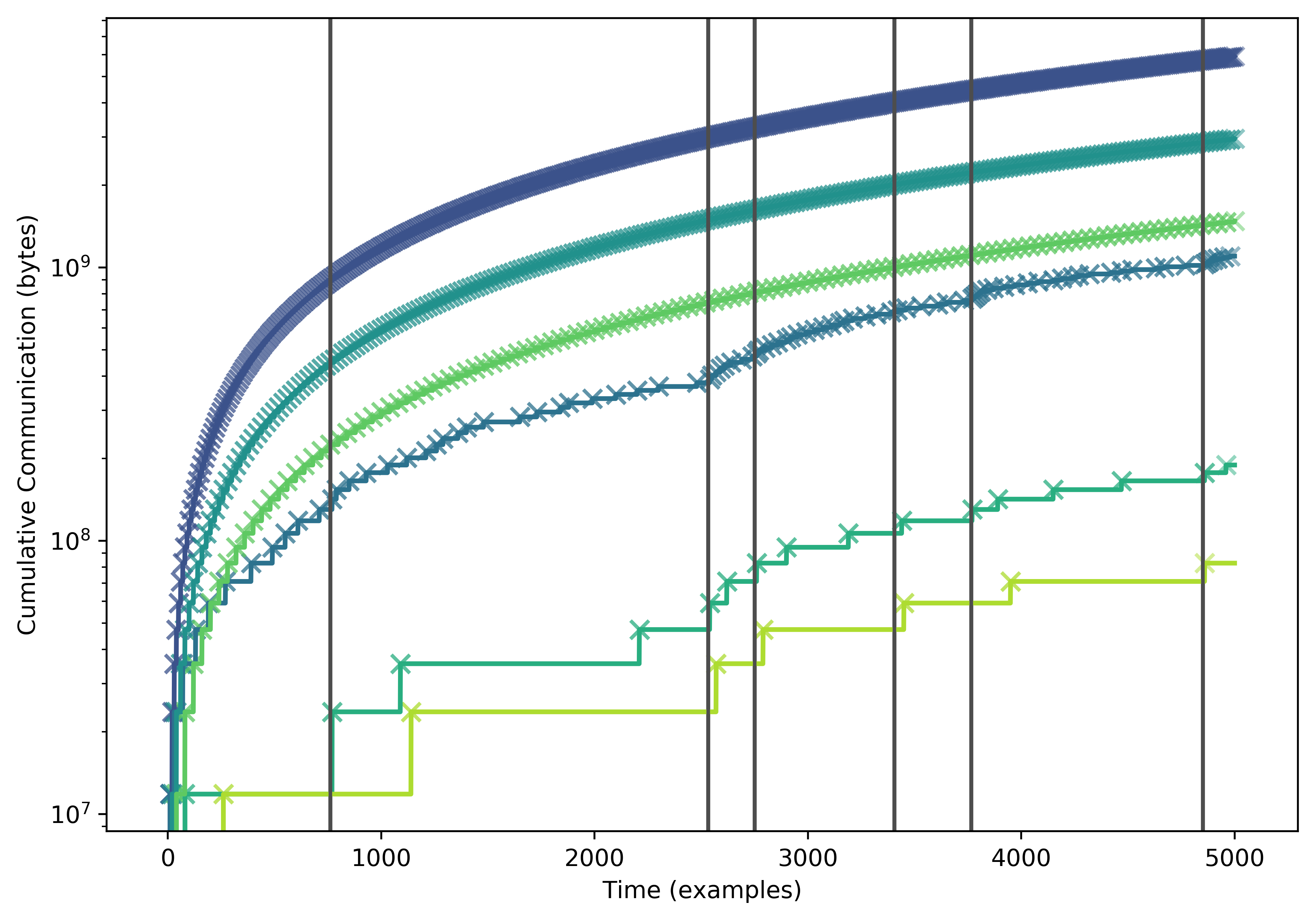}}
	\caption{Experiment with periodic and dynamic averaging protocols on $m=100$ learner after training on $5000$ samples per learner from a synthetic dataset with concept drifts (indicated by vertical lines in~\subref{fig:conceptDriftComm}).  
	\vspace{-0.3cm}
	}	
	\label{fig:conceptDriftExps}
\end{figure}
In order to investigate the behavior of dynamic and periodic averaging in this setting, we perform an experiment on a synthetic dataset generated by a random graphical model~\citep{bshouty/ml/2012}. Concept drifts are simulated by generating a new random graphical model. Drifts are triggered at random with a probability of $0.001$ per round. 

Figure~\ref{fig:conceptDriftExps}\subref{fig:conceptDriftPerf} shows
that in terms of predictive performance, dynamic and periodic averaging perform similarly. At the same time, dynamic averaging requires up to an order of magnitude less communication to achieve it. Examining the cumulative communication over time in Figure~\ref{fig:conceptDriftExps}\subref{fig:conceptDriftComm}, one can see that dynamic averaging communicates more after each concept drift and decreases communication until the next drift. This indicates that dynamic averaging invests communication when it is most impactful and can thereby save a substantial amount of communication in between drifts.  
%

\vspace{-0.2cm}
\subsubsection{Case Study on Deep Driving:}\label{ssec:deepdriving}
After having studied dynamic averaging in contrast to periodic approaches and FedAvg on MNIST and a synthetic 
dataset with concept drifts, we analyze how the suggested protocol performs in the realistic application scenario of in-fleet training for autonomous driving introduced in Section~\ref{sec:intro}. 
%
One of the approaches in autonomous driving is direct steering control of a constantly moving car via a neural network that predicts a steering angle given an input from the front view camera. 
Since one network fully controls the car this approach is termed \textbf{deep driving}. Deep driving neural networks can be trained on a dataset generated by recording human driver control and corresponding frontal view~\citep{bojarski2016end,fernando2017going,pomerleau1989alvinn}. 
%

For our experiments we use a neural network architecture suggested for deep driving by~\citet{bojarski2016end}. The learners are evaluated by their driving ability following the qualitative evaluation made by \citet{bojarski2016end} or \citet{pomerleau1989alvinn} as well as techniques used in the automotive industry. For that, we developed a custom loss together with experts for autonomous driving that takes into account the time the vehicle drives on track and the frequency of crossing road sidelines. 

Figure~\ref{fig:deep_driving_performance} shows the measurements of the custom loss against the cumulative communication. The principal difference from the previous experiments is the evaluation of the resulting models without taking into account cumulative training loss. 
All the resulting models as well as baseline models were loaded to the simulator and driven with a constant speed. The plot shows that each periodic 
\begin{wrapfigure}{r}{0.5\textwidth}
	\includegraphics[width=6cm]{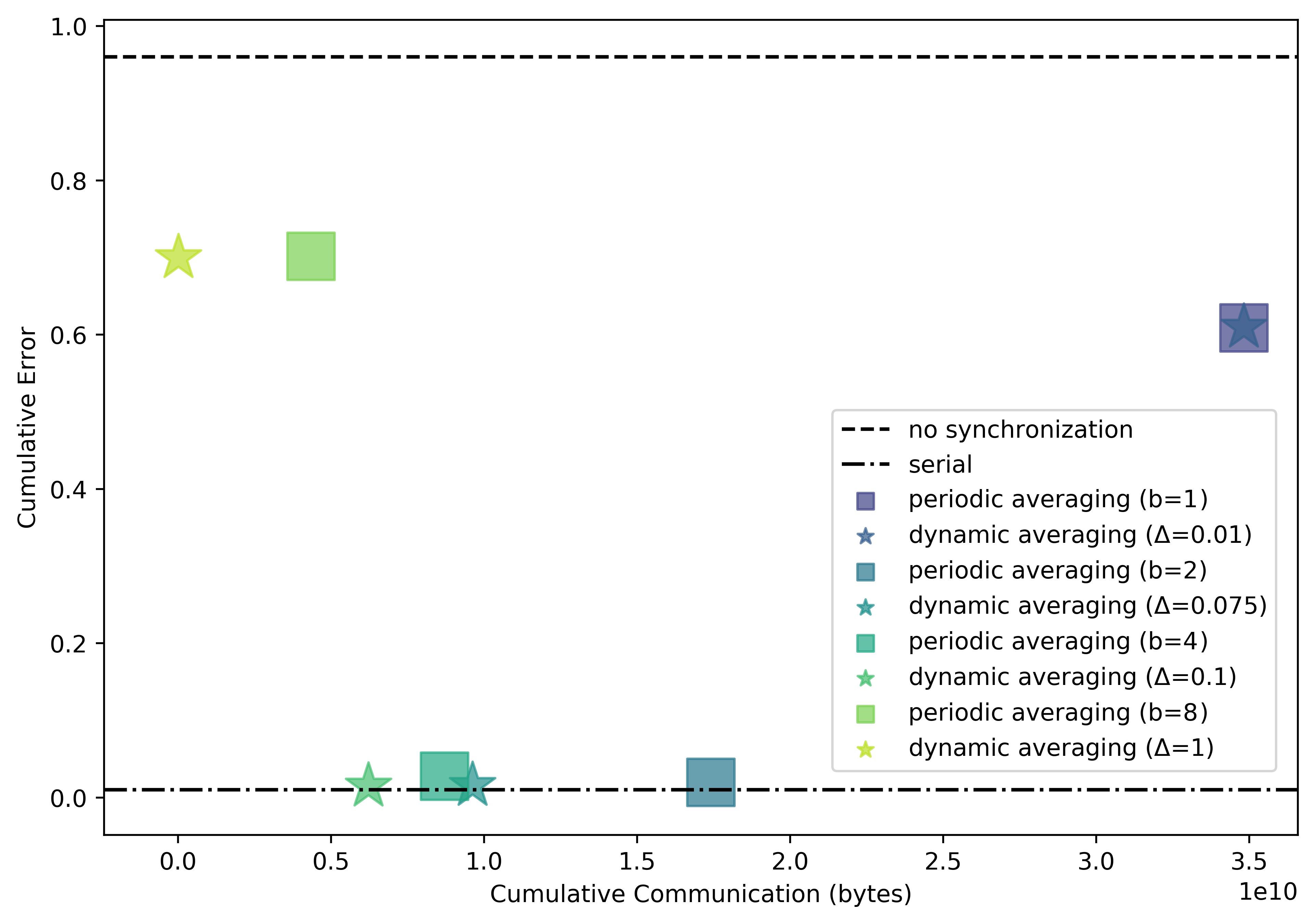}
	\caption{Performance in the terms of the custom loss for the models trained according to a set of communication protocols and baseline models.\vspace{-0.3cm}}
	\label{fig:deep_driving_performance}
\end{wrapfigure}
communication protocol can be outperformed by a dynamic protocol. 

Similar to our previous experiments, too little communication leads to bad performance, but for deep driving, very high communication ($\syncop_{b=1}$ and $\syncop_{\Delta=0.01}$) results in a bad performance as well. 
On the other hand, proper setups achieve performance similar to the performance of the serial model (e.g. dynamic averaging with $\Delta=0.1$ or $\Delta=0.3$). 
This raises the question, how much diversity is beneficial in-between averaging steps and how diverse models should be initialized. We discuss this question and other properties of dynamic averaging in the following section.

\section{Discussion}
\label{sec:discussion}
A popular class of parallel learning algorithms is based on stochastic gradient descent, both in convex and non-convex learning tasks. As for all gradient-based algorithms, the gradient computation can be parallelized `embarrassingly'~\citep{moler1986matrix} easily. For convex problems, the best so far known algorithm, in terms of predictive performance, in this class~\citep{shamir2014distributed} is the distributed mini-batch algorithm \citep{dekel/jmlr/2012}. 
For the non-convex problem of training (deep) neural networks,~\citet{mcmahan2017communication} have shown that periodic averaging performs similar to the mini-batch algorithm. Section~\ref{sec:theory} substantiates these results from a theoretical perspective. Sub-sampling learners in each synchronization allows to further reduce communication at the cost of a moderate loss in predictive performance. 

Note that averaging models, similar to distributed mini-batch training, requires a common architecture for all local models since the goal is to jointly train a single global model distributedly using observations from local data streams---which also sets it apart from ensemble methods.

For the convex case,~\citet{kamp2016kernels} have shown that dynamic averaging retains the performance of periodic averaging and certain serial learning algorithms (including SGD) with substantially less communication. Section~\ref{sec:theory} proves that these results are applicable to the non-convex case as well. Section~\ref{sec:exps} indicates that these results also hold in practice and that dynamic averaging indeed outperforms periodic averaging, both with and without sub-sampling of learners. 
This advantage is even amplified in the presence of concept drifts. Additionally, dynamic averaging is a black-box approach, i.e., it can be applied with arbitrary learning algorithms (see Appendix~\ref{app:ssec:black_box} for a comparison of using dynamic averaging with SGD, ADAM, and RMSprop). 

However, averaging models instead of gradients has the disadvantage of being susceptible to outliers. That is, without a bound on the quality of local models, their average can be arbitrarily bad~\citep{shamir2014distributed, kamp2017effective}. More robust approaches are computationally expensive, though, e.g., the geometric median~\citep{feng2017outlier}. Others are not directly applicable to non-convex problems, e.g., the Radon point~\citep{kamp2017effective}. Thus, it remains an open question whether robust methods can be applied to decentralized deep learning. 

Another open question is the choice of the divergence threshold $\Delta$ for dynamic averaging. The model divergence depends on the expected update steps (e.g., in the case of SGD on the expected norm of gradients and the learning rate), but the threshold is not intuitive to set. A good practice is to 
\begin{wrapfigure}{r}{0.5\textwidth}
	\includegraphics[width=0.5\textwidth]{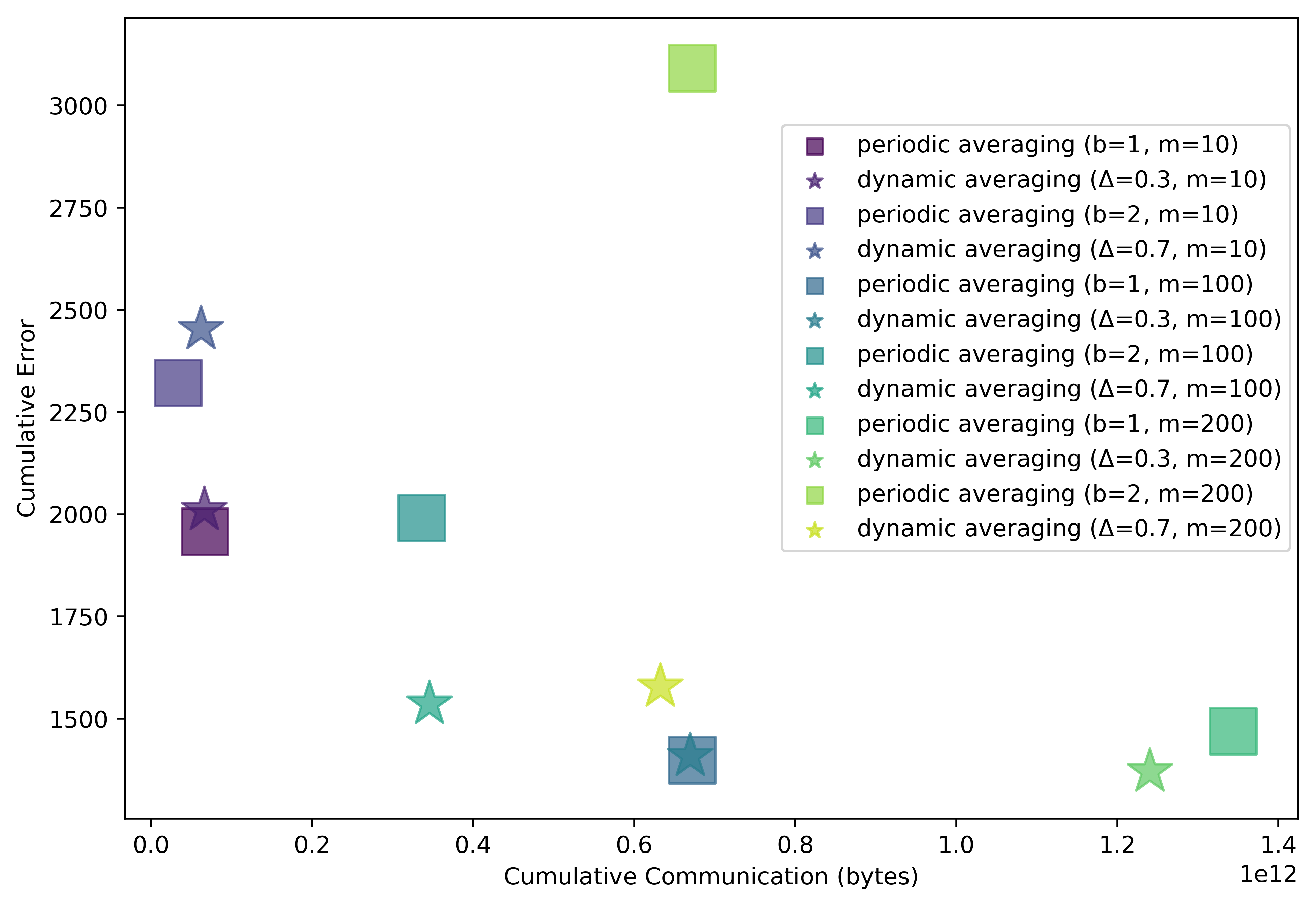}
	\centering
	\caption{
		Cumulative loss and cumulative communication of learning protocols for a different amount of learners. Training is performed on MNIST for $2$, $20$ and $40$ epochs for $m=10$, $m=100$, $m=200$ setups correspondingly.\vspace{-0.3cm}}
	\label{fig:10_100_200_performance}
\end{wrapfigure}
optimize the parameter for the desired trade-off between predictive performance and communication on a small subset of the data. It is an interesting question whether the parameter can also be adapted during the learning process in a theoretically sound way.

In dynamic averaging, the amount of communication not only depends on the actual divergence of models, but also on the probability of local violations. Since the local conditions can be violated without the actual divergence crossing the threshold, these false alarms lead to unnecessary communication. The more learners in the system, the higher the probability of such false alarms. In the worst case, though, dynamic averaging communicates as much as periodic averaging. Thus, it scales at least as well as current decentralized learning approaches~\citep{mcmahan2017communication,jiang2017collaborative}. Moreover, using a resolution strategy that tries to balance violations by communicating with just a small number of learners partially compensates for this problem. Indeed, experiments on the scalability of the approach show that dynamic averaging scales well with the number of learners (see Figure~\ref{fig:10_100_200_performance} and Appendix~\ref{app:ssec:scale_out} for details).

\begin{figure}[t]
	\centering
	\subfigure[static averaging protocols]{\label{fig:init_1}\centering\includegraphics[width=5.4cm]{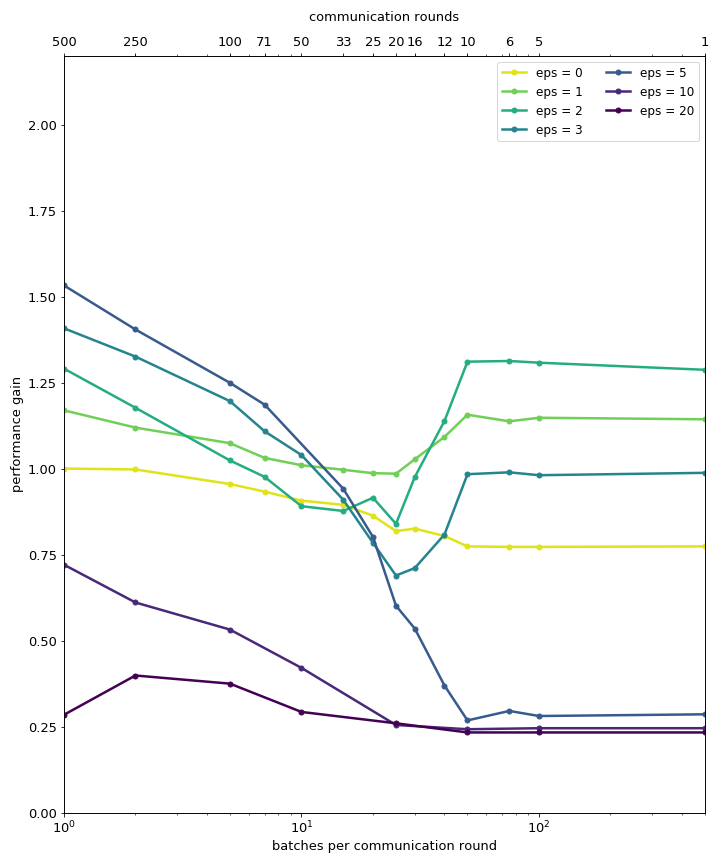}}
	\subfigure[dynamic averaging protocols]{\label{fig:init_2}\centering\includegraphics[width=5.4cm]{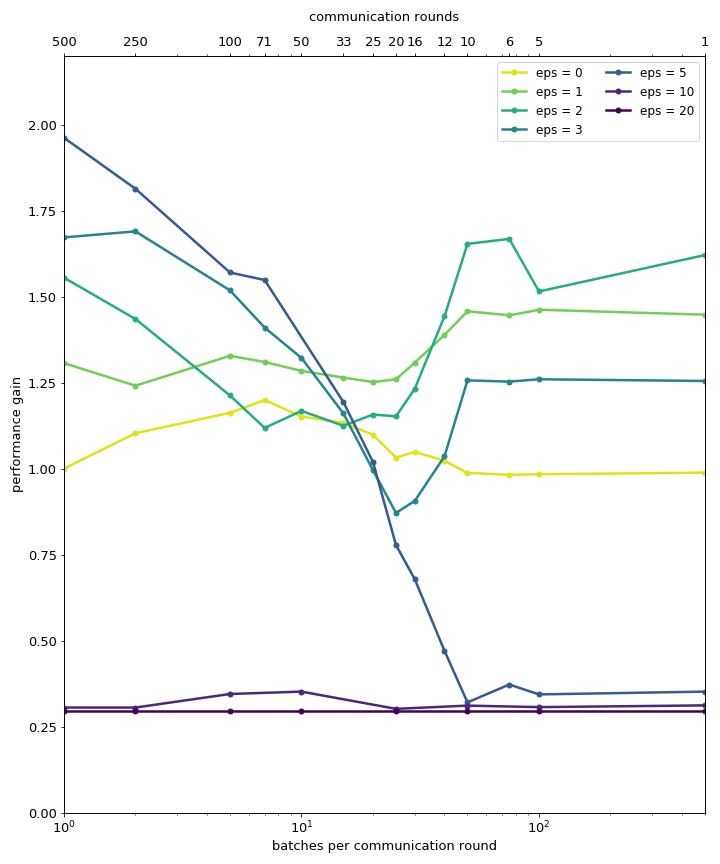}}
	\caption{Relative performances of averaged models on MNIST obtained from various heterogeneous model initializations parameterized by $\epsilon$ and various $b\in\N$. All averaged model performances are compared to an experiment with homogeneous model initializations ($\epsilon = 0$) and $b = 1$.\vspace{-0.3cm}}
	\label{fig:init}
\end{figure}
A general question when using averaging is how local models should be initialized. \citet{mcmahan2017communication} suggest using the same initialization for all local models and report that different initializations deteriorate the learning process when models are averaged only once at the end. Studying the transition from homogeneously initialized and converging model configurations to heterogeneously initialized and failing ones reveals that, surprisingly, for multiple rounds of averaging different initializations can indeed be beneficial. Figure~\ref{fig:init} shows the performances of dynamic and periodic averaging for different numbers of rounds of averaging and different levels of inhomogeneity in the initializations. The results confirm that for one round of averaging, strongly inhomogeneous initializations deteriorate the learning process, but for more frequent rounds of averaging mild inhomogeneity actually improves training. For large heterogeneities, however, model averaging fails as expected. This raises an interesting question about the regularizing effects of averaging and its potential advantages over serial learning in case of non-convex objectives.

\section{Conclusion}
\label{sec:conclusion}
In decentralized deep learning there is a natural trade-off between learning performance and communication. Averaging models periodically allows to achieve a high predictive performance with less communication compared to sharing data. The proposed dynamic averaging protocol achieves similarly high predictive performance yet requires substantially less communication. At the same time, it is adaptive to concept drifts. The method is theoretically sound, i.e., it retains the loss bounds of the underlying learning algorithm using an amount of communication that is bound by the hardness of the learning problem.

\subsubsection*{Acknowledgements}
This research has been supported by the Center of Competence Machine Learning Rhein-Ruhr (ML2R).

\small
\bibliographystyle{splncsnat}
\bibliography{bibliography}

\newpage
\newpage
\appendix
\label{sec:appendix}

\section{Details on the Empirical Evaluation}
\label{app:sec:exps}
In this section we provide additional details on the experimental setup and the empirical evaluation presented in Section~\ref{sec:exps}.

\subsection{Experiments on MNIST}
\label{app:sec:exps:mnist}

For the experiments with MNIST dataset we have chosen a neural network with two convolutional layers, one max pooling layer, and two dense layers with final softmax activation. The overview of the network layers and the amount of parameters are presented in Table~\ref{tab:mnist_network}. To make the results comparable to those of~\citet{mcmahan2017communication} we chose $\totalLearners=100$ learners, each training a network with roughly $10^6$ weights. This distributed training is performed for three setups of both periodic and dynamic averaging (summarized in Table~\ref{mnist100_table}), as well as for the two baselines, serial and nosync. 

The employed loss function for training is categorical crossentropy and the learning algorithm is mini-batch SGD (as it is implemented in the Keras library~\citep{chollet2015keras}).
\begin{table}[h]
\begin{center}
    \medskip
    \begin{tabular}{ cccp{5cm} }
    \hline
    \textbf{Layer Type} & \textbf{Output Shape} & \textbf{\#Weights} \\ \hline
    Conv2D & (26, 26, 32) & 320 \\ 
    Conv2D & (24, 24, 64) & 18\,496 \\ 
    MaxPooling2D & (12, 12, 64) & 0 \\ 
    Dropout & (12, 12, 64) & 0 \\ 
    Flatten & (9216) & 0 \\ 
    Dense & (128) & 1\,179\,776 \\ 
    Dropout & (128) & 0 \\ 
    Dense & (10) & 1\,290 \\ \hline
    \multicolumn{2}{ c }{\textbf{Total}} & 1\,199\,882 \\ 
    \end{tabular}
    \caption{The architecture of the Convolutional Neural Network used for MNIST dataset. Printout of the parameters made via Keras library.}
    \label{tab:mnist_network}
\end{center}
\end{table}

The parameters of the learning algorithm are optimized on a separate dataset. For the serial learner, the optimal parameters are a batch size of $B=10$ samples and a learning rate of $\eta=0.1$. With these parameters the serial model reaches an accuracy of $0.99$ (calculated as an incremental value during the online training time) which is competitive on the MNIST dataset. For the distributed learners, the batch size is again $B=10$, but the learning rate is higher with $\eta=0.25$. 

For the distributed setups, each of the $\totalLearners=100$ learners observes $14000$ examples. At the same time, the serial baseline observes the same amount of examples as the entire distributed learning system, i.e., $100 \cdot 14 000$ samples of the $70 000$ data points. Thus, while each local learners observes only a fraction of the dataset, the serial learner trains for $20$ epochs over the data. This explains the substantially lower cumulative loss of the serial baseline compared to the distributed protocols. However, the serial baseline requires centralization of all training data which is often infeasible in decentralized application scenarios. Moreover, this data needs to be centrally processed, resulting in a runtime that is roughly $\totalLearners$-times higher than that of the distributed approaches.
\begin{table}[h]
	\begin{center}
		\begin{tabular}{ ll }
			\multicolumn{2}{ c }{\textbf{protocol configurations}} \\
			\textbf{type} & \textbf{parameters} \\ \hline
			\multirow{3}{*}{periodic protocol ($\syncop_b$)} & $b=1$ \\ 
			 & $b=2$ \\
			 & $b=4$ \\ 
			 \\
			\multirow{3}{*}{dynamic protocol ($\syncop_{\Delta,b}$)} &  $b = 1, \Delta = 0.3$ \\
			& $b = 1, \Delta = 0.7$ \\
			& $b = 1, \Delta = 1.0$ \\ 
		\end{tabular}
	\end{center}
	\caption{Overview of the different communication protocol configurations used for MNIST experiment. Except for protocol parameters $\Delta$ and $b$, all other parameters are kept constant between configurations.}
	\label{mnist100_table}
\end{table}

\begin{figure}[h]
	\centering
	\subfigure{
		\label{fig:communication_selected}
		\includegraphics[height=4cm]{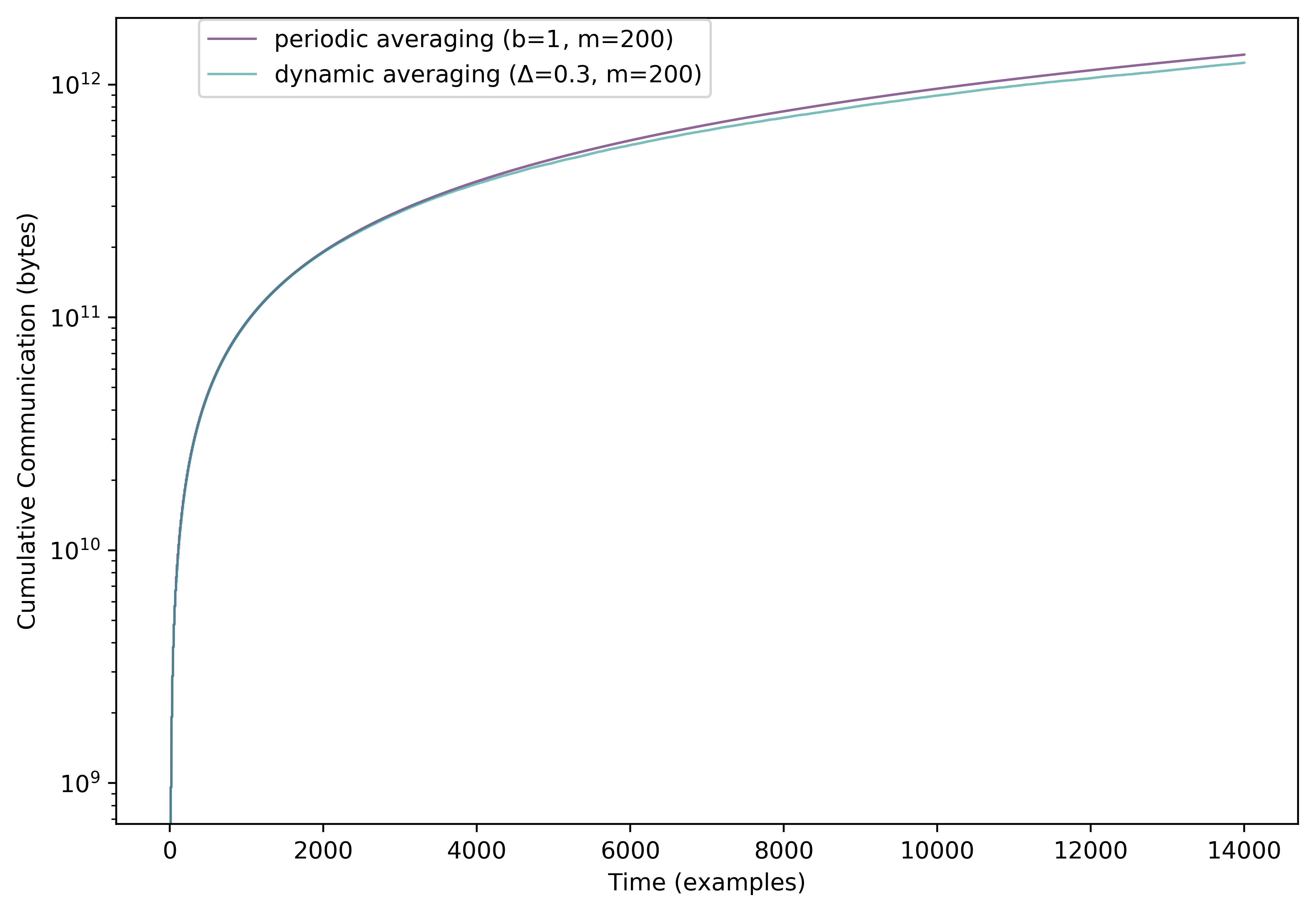}}
	\hfill
	\subfigure{
		\label{fig:loss_selected}
		\includegraphics[height=4cm]{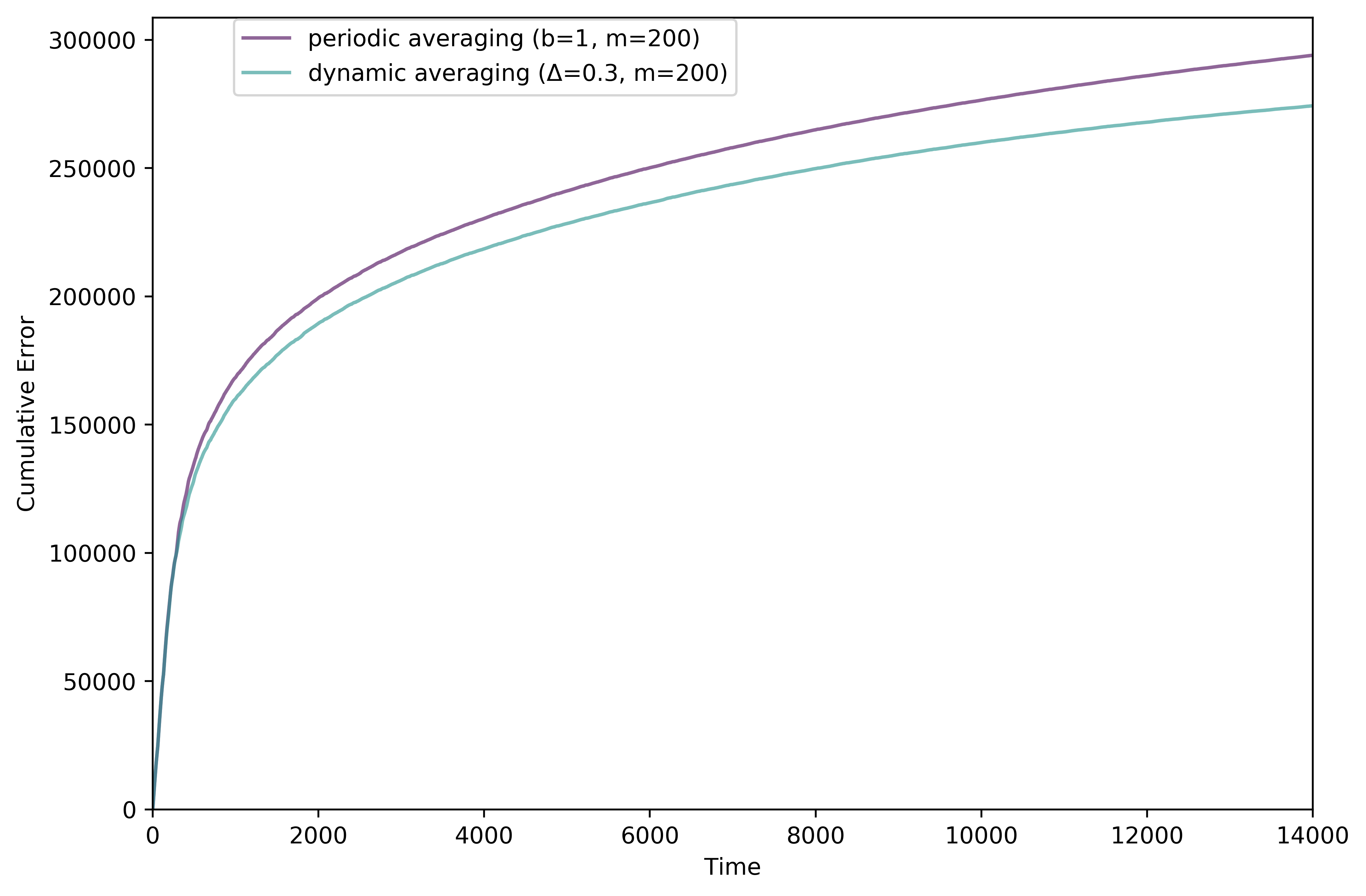}}
	\caption{
		\subref{fig:communication_selected} The cumulative communication development during the training on MNIST dataset for $40$ epochs using dynamic averaging ($\Delta=0.3, b=1$) and periodic averaging ($b=1$).
		\subref{fig:loss_selected} The cumulative loss development during the training on MNIST dataset for $40$ epochs with one dynamic and one periodic protocols.}
	\label{fig:loss_communcation_tradeoff_selected}
\end{figure} 

As a solution, periodic averaging allows to sacrifice predictive performance to allow decentralized computation in such a way that the more the protocol communicates, the more performance is maintained. At the same time, for each setup of periodic averaging, a parameter $\Delta$ of dynamic averaging can be found such that it has similar predictive performance but requires less communication.
As an example of this, in Figure~\ref{fig:loss_communcation_tradeoff_selected} we present the cumulative communication and error over time of two similarly performing protocols, using $\syncop_{\Delta=0.3}$, respectively $\syncop_{b=1}$. During first $500$ examples, when both protocols suffer a lot of loss, dynamic averaging invests more communication, resulting in faster convergence and thus a lower increase in cumulative error. After the first $3000$ examples, dynamic averaging reduces the amount of communication without suffering more loss. This leads eventually to both smaller cumulative error and cumulative communication.

\subsection{Comparison with FedAvg}
\label{app:sec:exps:fedsgd}
In Table~\ref{fedsgd_table}, we provide an overview of the dynamic averaging and FedAvg protocol configurations used in Section~\ref{ssec:fedsgd}. The experiment runs with $m=30$ learners, mini batches of size $B=10$ and $b=5$. The FedAvg parameter $C\in (0,1]$~\cite{mcmahan2017communication} is the fraction of learners involved in a particular model synchronization and their parameter $E$, the number of local batches, corresponds to the parameter $b$. Each learner is trained on $8000$ training examples. Accuracy is calculated on the last $100$ training examples.
\begin{table}[ht]
	\begin{center}
		\begin{tabular}{ ll }
			\multicolumn{2}{ c }{\textbf{protocol configurations}} \\
			\textbf{type} & \textbf{parameters} \\ \hline
			periodic protocol ($\syncop_{b}$) & $b=5$ \\ 
			\\
			\multirow{4}{*}{dynamic protocol ($\syncop_{\Delta,b}$)} & $b=5, \Delta = 0.1$ \\
			& $b=5, \Delta = 0.2$ \\
			& $b=5, \Delta = 0.4$ \\
			& $b=5, \Delta = 0.6$ \\
			& $b=5, \Delta = 0.8$ \\ 
			\\
			\multirow{3}{*}{FedAvg ($\syncop_{\rm{FedAvg}}$)} & $b=5, C = 0.3$ \\
			& $b=5, C = 0.5$ \\
			& $b=5, C = 0.7$ \\ 
		\end{tabular}
	\end{center}
	\caption{Overview of the different communication protocol configurations for the comparison with FedAvg. Except for protocol parameters $\Delta$ and $C$, all other parameters are kept constant between configurations.}
	\label{fedsgd_table}
\end{table}
\begin{figure}[h]
\centering
\subfigure{\label{fig:fedsgd_errorCom_1_app}\includegraphics[width=6cm]{images/MNIST_CNN_fedsgd/comparison_fedsgd_communication.png}}
\hfill
\subfigure{\label{fig:fedsgd_errorCom_2_app}\includegraphics[width=6cm]{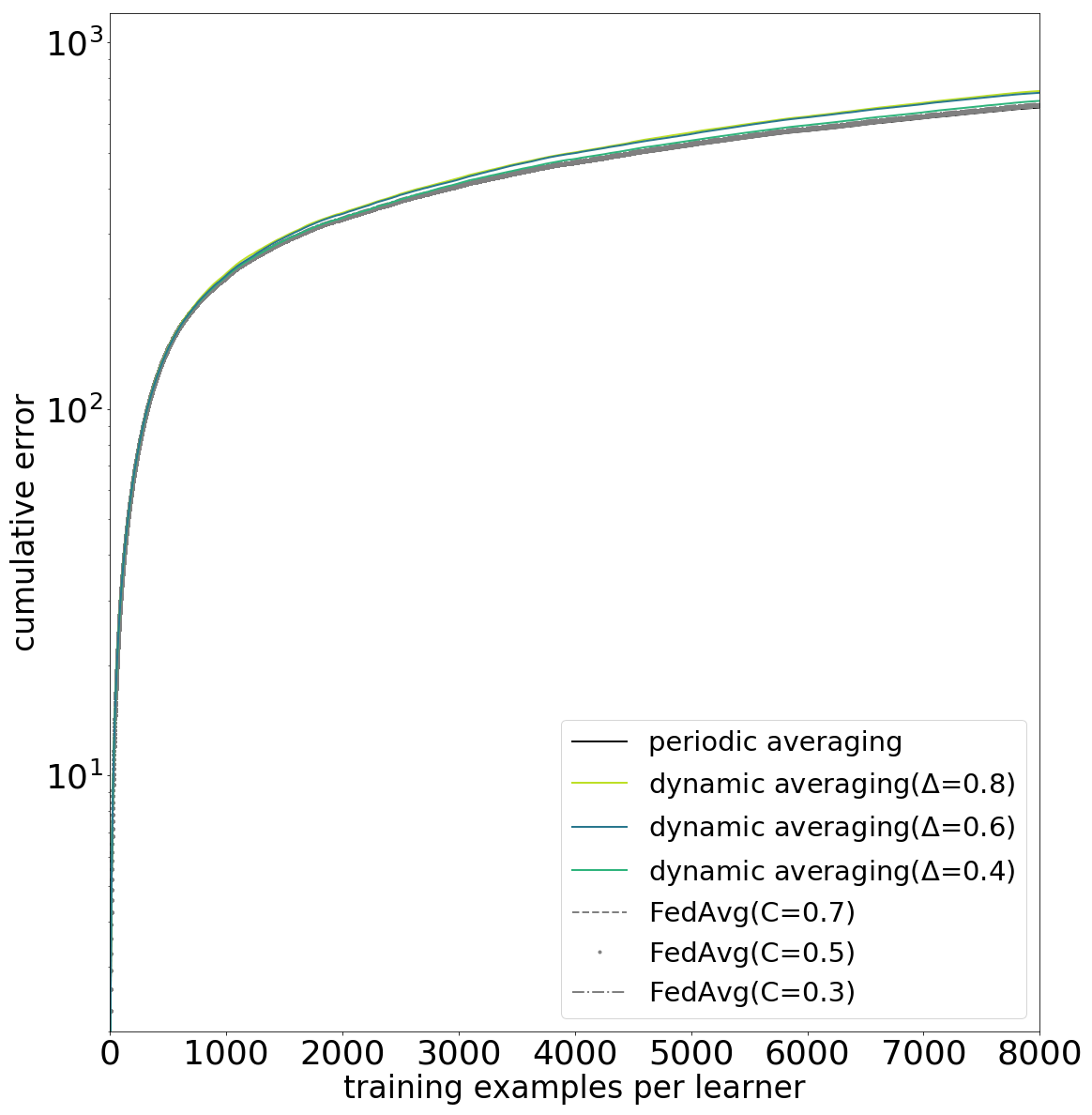}}
\caption{Evolution of \subref{fig:fedsgd_errorCom_1_app} cumulative communication and \subref{fig:fedsgd_errorCom_2_app} cumulative error during training for different dynamic averaging and FedAvg protocols.}
\label{fig:fedsgd_errorCom_app}
\end{figure}

\begin{figure}[h]
\centering
\subfigure{\label{fig:errorCom_comparison_periodicAveraging}\includegraphics[width=6cm]{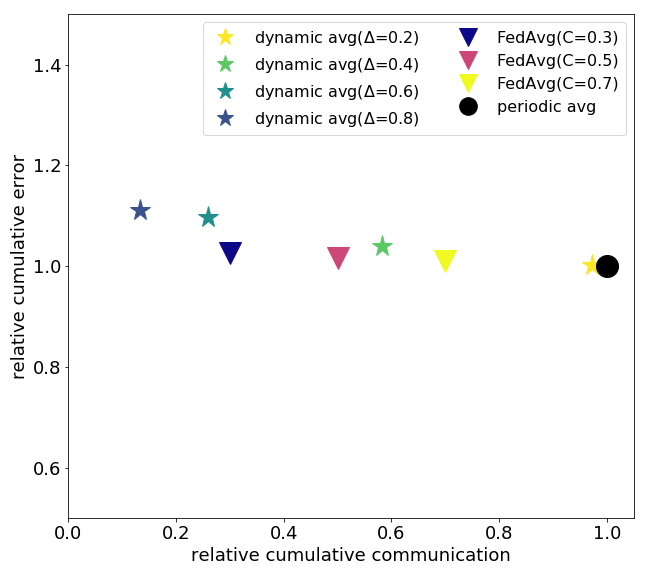}}
\hfill
\subfigure{\label{fig:accCom_comparison_periodicAveraging}\includegraphics[width=6cm]{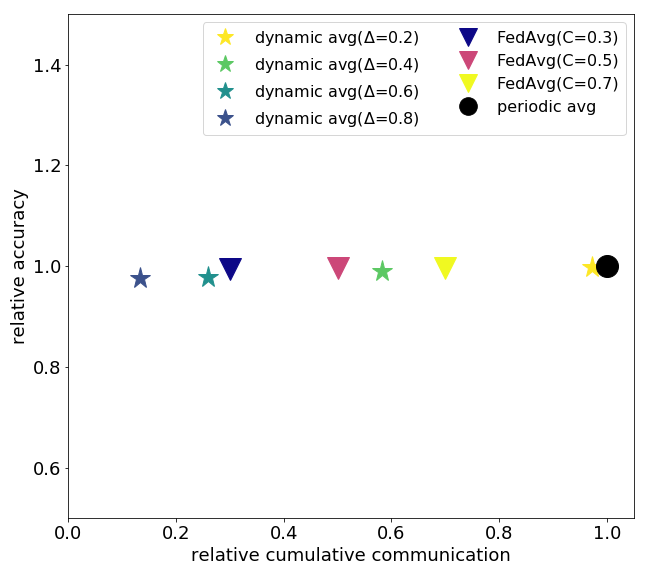}}
\caption{Evaluation of the dynamic averaging and FedAvg protocols relative to periodic averaging \subref{fig:errorCom_comparison_periodicAveraging} in terms of cumulative error and communication, \subref{fig:accCom_comparison_periodicAveraging} in terms of accuracy and communication.}
\label{fig:errorAccCom_comparison_periodicAveraging}
\end{figure}

The loss cumulated by $\syncop_{\Delta=0.6}$ and $\syncop_{\Delta=0.8}$ during training is only very slightly larger than those of the other $\dynavg$ and $\syncop_{\rm{FedAvg}}$. Figure~\ref{fig:fedsgd_errorCom_app} illustrates this trade-off in greater detail, showing the development of cumulative communication and cumulative loss over time. Both kinds of protocols lead to a significant reduction of communication in comparison to $\syncop_{b=5}$ (see Figure~\ref{fig:errorAccCom_comparison_periodicAveraging}), in the case of $\syncop_{\Delta=0.8}$ even more than 80\%. Both, $\dynavg$ and $\syncop_{\rm{FedAvg}}$, go along with slight increases of the cumulative loss (Figure~ \ref{fig:errorAccCom_comparison_periodicAveraging}\subref{fig:errorCom_comparison_periodicAveraging}) and slight decreases of model accuracy (Figure~\ref{fig:errorAccCom_comparison_periodicAveraging}\subref{fig:accCom_comparison_periodicAveraging}). 

\subsection{Adaptivity to Concept Drift}
\label{app:sec:exps:conceptDrift}
To assess the adaptivity of dynamic averaging, we compare it to periodic averaging on a synthetic binary classification dataset with samples from $\R^d$, with $d=50$, generated using a random graphical model~\citep{bshouty/ml/2012}. Concept drifts are triggered at random with a probability of $0.001$. A concept drift is simulated by generating a new graphical model.
\begin{table}[h]
	\begin{center}
		\medskip
		\begin{tabular}{ cccp{5cm} }
			\hline
			\textbf{Layer Type} & \textbf{Output Shape} & \textbf{\#Weights} \\ \hline
			Dense & (256) & 13\,056 \\ 
			Dropout & (256) & 0 \\ 
			Dense & (64) & 16\,448 \\ 
			Dense & (1) & 65 \\ \hline
			\multicolumn{2}{ c }{\textbf{Total}} & 29\,569 \\ 
		\end{tabular}
		\caption{The architecture of the fully connected network used for the concept drift experiment.}
		\label{tab:bshouty_network}
	\end{center}
\end{table}
Table \ref{tab:bshouty_network} provides an overview of the network layers and the amount of neurons used.
\begin{table}[h]
	\begin{center}
		\begin{tabular}{ ll }
			\multicolumn{2}{ c }{\textbf{protocol configurations}} \\
			\textbf{name} & \textbf{parameters}\\
			\hline
			\multirow{3}{*}{periodic protocol $\syncop_{b}$} & $b=1$ \\ 
			& $b=2$ \\
			& $b=4$ \\ 
			\\
			\multirow{3}{*}{dynamic protocol $\syncop_{\Delta,b}$} & $b=1, \Delta = 0.3$ \\
			& $b=1, \Delta = 0.7$ \\
			& $b=1, \Delta = 1.0$ \\ 
		\end{tabular}
	\end{center}
	\caption{Overview of the different communication protocol configurations for the analysis of concept drift. Except for protocol parameters $\Delta$ and $b$, all other parameters are kept constant between configurations.}
	\label{conceptDrift_table}
\end{table}
\begin{figure}[h]
\centering
\subfigure[cumulative loss]{\label{fig:conceptDriftError_app}\includegraphics[height=4cm]{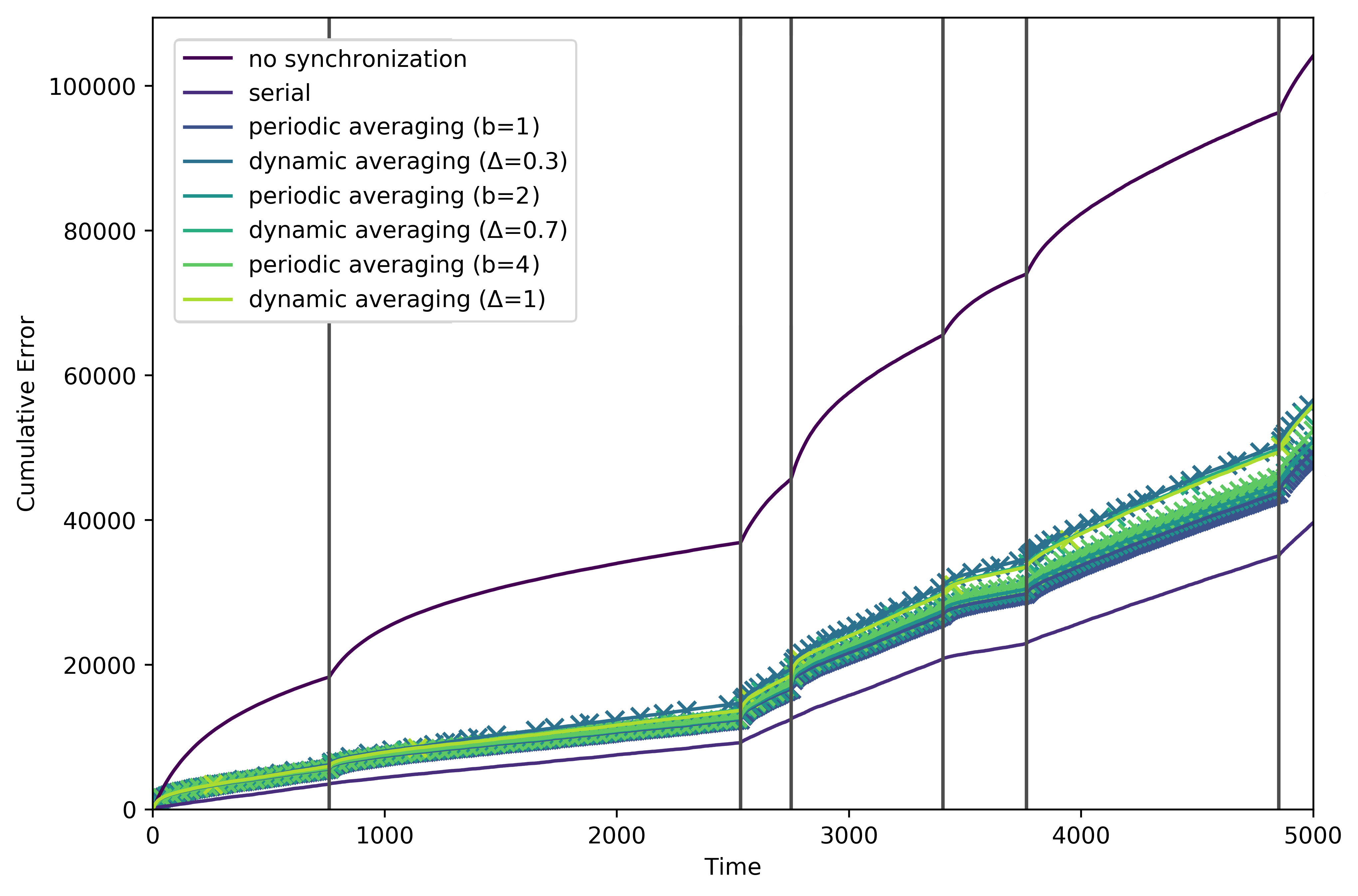}}
\hfill
\subfigure[cumulative communication]{\label{fig:conceptDriftComm_app}\includegraphics[height=4cm]{images/BSHOUTY_conceptDrift/communication_time.png}}
\caption{Development of cumulative loss and communication during training on $5000$ examples per node from a synthetic dataset with concept drifts (indicated by vertical lines).}
\label{fig:loss_communcation_tradeoff_app}
\end{figure}
Figure~\ref{fig:loss_communcation_tradeoff_app}, that shows the development of cumulative loss and communication over time, confirms that dynamic averaging adapts the amount of communication: right after a concept drift the number of synchronizations (indicated by a cross mark) is high and decreases as soon as the instantaneous loss of the learners---i.e., observable as the growth of the cumulative error---decreases.

\subsection{Deep Driving Experiments}
\label{app:sec:exps:dd}
Deep driving is an example for an autonomous driving functionality that can be learned in-fleet. That is, every vehicle continuously trains a local model based on its observations in shadow mode. This requires to infer the correct output locally for training. For that, the correct output for training is inferred from a human driver by mimicking his driving behavior using a frontal view camera as input~\citep{bojarski2016end,fernando2017going,pomerleau1989alvinn}. 

Using only vehicles in a specific region, the data seen by an individual local learner has in good approximation a low variability and heterogeneity, because people driving cars tend to stay close to their base. Thus, data from cars 
from a similar region can be assumed to be fairly homogeneously distributed. 

In order to mimic this scenario for our experiments, we record human driving behavior in a simulation for multiple drivers on a single track.
However, data of cars from different base locations, collected at different times or in different cars will underlie different (local) approximations of the actual distribution. This actual distribution has in contrast to its local approximations a large variability and heterogeneity, even if one considers only the minor set of tasks solved by machine learning. Thus, for in-fleet training over various regions the data cannot be assumed iid anymore, but empirical results~\citep{mcmahan2017communication} and recent extensions to the federated learning~\citep{smith2017federated} suggest that the approach is capable of handling non-iid data. Analyzing these and other challenges for in-fleet learning is an interesting direction for future work.

In the following, we provide details on the experimental setup.
The architecture of the layers of the deep driving network can be seen in Table \ref{tab:dd_network}.
\begin{table}[ht]
	\begin{center}
		\medskip
		\begin{tabular}{ cccp{5cm} }
			\hline
			\textbf{Layer Type} & \textbf{Output Shape} & \textbf{\#Weights} \\ \hline
			Conv2D & (32, 158, 24) & 1\,824 \\ 
			Conv2D & (14, 77, 36) & 21\,636 \\ 
			Conv2D & (5, 37, 48) & 43\,248 \\ 
			Conv2D & (3, 35, 64) & 27\,712 \\ 
			Conv2D & (1, 33, 64) & 36\,928 \\ 
			Flatten & (2\,112) & 0 \\ 
			Dense & (100) & 211\,300 \\ 
			Dense & (50) & 5\,050 \\ 
			Dense & (10) & 510 \\ 
			Dense & (1) & 11 \\ \hline
			\multicolumn{2}{ c }{\textbf{Total}} & 348\,219 \\ 
		\end{tabular}
		\caption{The architecture of the Convolutional Neural Network used for deep driving. Printout of the parameters made via Keras library.}
		\label{tab:dd_network}
	\end{center}
\end{table}
The employed loss function is squared error as it is implemented in Keras ``mean\_squared\_error''. 
The network is optimized using mini-batch SGD.
The experiments were run with training batch size of $B=10$ and learning rate $\eta=0.1$ both for the baselines and local learners. 
The input to the network is the front camera view from a car driven in a simulator\footnote{\url{https://github.com/udacity/self-driving-car-sim}}.
The output of the network is a steering angle that allows to control the car in the autonomous mode in the simulator when the speed is kept on a constant level.

During the evaluation a trained model has been loaded to drive the car in the autonomous regime in the simulator. The time that the model is able to control the car without going off the road or crashing is measured. The longest time $t_{max}$ is the time for the model that is able to keep going for $2$ laps on the track or the maximum time of all the models in one experiment. Also the amount of times the car has touched the sideline of the road is counted together with the time duration while car is still on the sideline $t_{line}$. Then a frequency $c$ of sideline crossings is calculated in a form $\frac{\# crossings}{t}$, where $t$ is the time before going off road or crash, and the maximal frequency among all the models is assigned to $c_{max}$. The overall formula for the custom loss is:
\[
L_{dd} = \lambda \frac{t_{max} - t}{t_{max}} + \mu  \frac{c}{c_{max}} + (1 - \mu - \lambda)  \frac{t_{line}}{t}
\]
\begin{wrapfigure}{r}{0.5\textwidth}
	\includegraphics[width=0.5\textwidth]{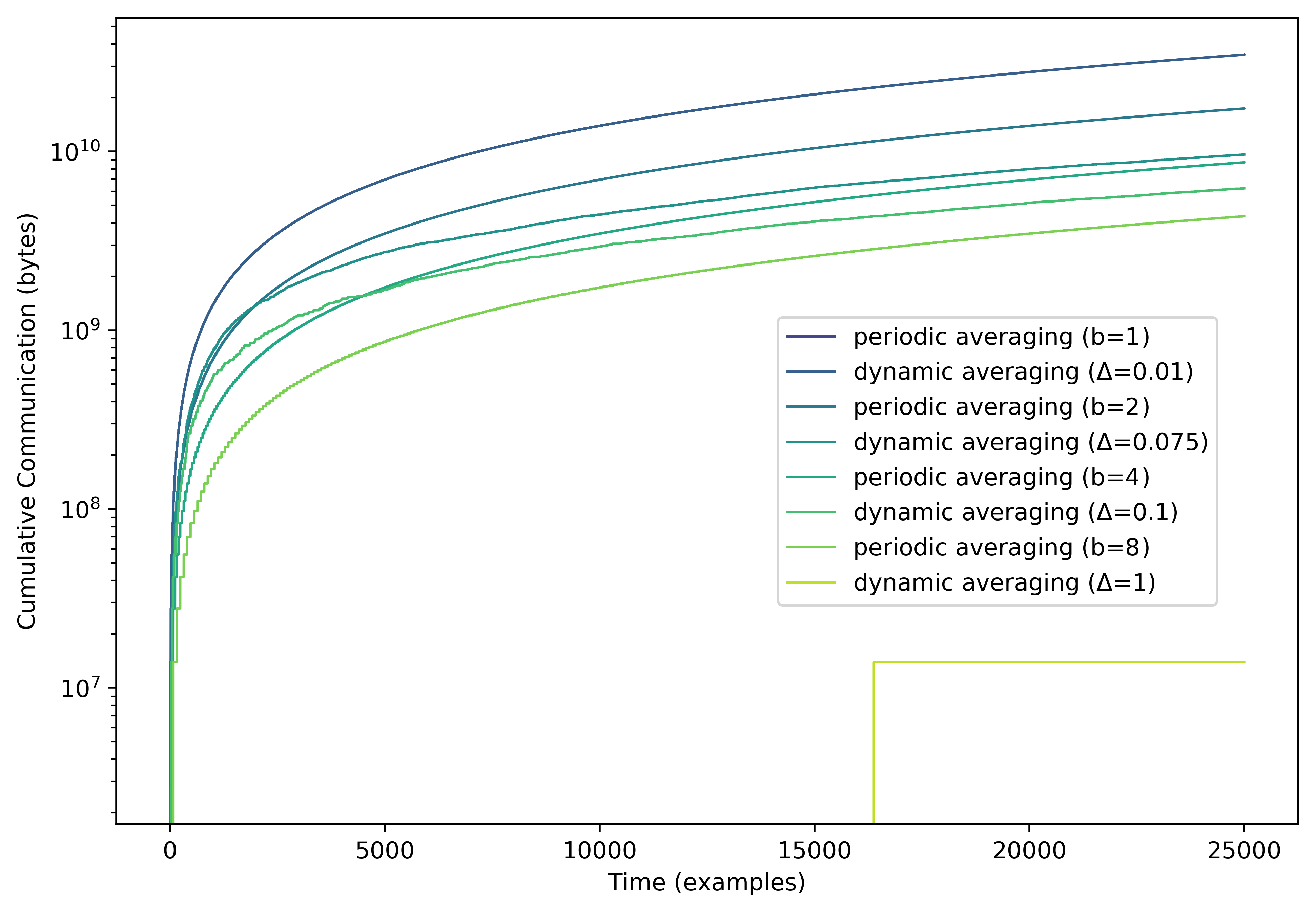}
	\caption{Evaluation of cumulative communication for the deep driving experiment.}
	\label{fig:deep_driving_comm}
\end{wrapfigure}
where $t$ is the time that the model is able to drive on the road, $\lambda, \mu \in [0;1]$ are weighting coefficients.
For the experiment $\lambda=0.8$ and $\mu=0.15$ were used. The amount of examples shown to each of the $m=10$ learners is $25000$, i.e., the dataset was shown to the serial learner approximately $5$ times (the overall size of the dataset is $\approx 48000$). The list of the considered communication protocols can be seen from Table~\ref{dd_table}.

In Section~\ref{sec:exps} we have shown that dynamic and periodic averaging can achieve similar driving performance as the serial baseline. For that, dynamic averaging requires substantially less communication than periodic. Examining the evolution of cumulative communication (see Figure~\ref{fig:deep_driving_comm}) shows that---similar to the results on MNIST---dynamic averaging invests a large amount of communication in the beginning and then considerably reduces communication. This pattern is most apparent for dynamic averaging with $\Delta=0.1$.

\begin{table}[ht]
	\begin{center}
		\begin{tabular}{ ll }
			\multicolumn{2}{ c }{\textbf{protocol configurations}} \\
			\textbf{type} & \textbf{parameters} \\ \hline
			\multirow{4}{*}{periodic protocol ($\syncop_{b}$)} & $b=1$ \\ 
			& $b=2$ \\
			& $b=4$ \\
			& $b=8$ \\ 
			\\
			\multirow{4}{*}{dynamic protocol ($\syncop_{\Delta, b}$)} & $b=1, \Delta = 0.01$ \\
			& $b=1, \Delta = 0.05$ \\
			& $b=1, \Delta = 0.1$ \\
			& $b=1, \Delta = 0.3$ \\ 
		\end{tabular}
	\end{center}
	\caption{Overview of the different communication protocol configurations used for deep driving experiment. Except for protocol parameters $\Delta$ and $b$, all other parameters are kept constant between configurations.}
	\label{dd_table}
\end{table}

\subsection{Black-Box Algorithms Evaluation}
\label{app:ssec:black_box}
In comparison with distributed mini-batch SGD \citep{dekel/jmlr/2012,chen2016revisiting} our approach allows to treat the optimization algorithm as a black-box, i.e., dynamic synchronization protocols achieve performance similar to $\dynProt=(\uprule^{mSGD}, \syncop_{\Delta,b})$ with various different optimization algorithms. Particularly, we conducted experiments with the ADAM optimizer \citep{kingma2014adam} and RMSprop \citep{tieleman2012lecture}. The results show that the advantage of dynamic averaging over periodic also holds for those learning algorithms (see Figure~\ref{fig:blackbox}).

\begin{figure}
\centering
\subfigure[$\uprule^{Adam}$]{
		\includegraphics[width=5.5cm]{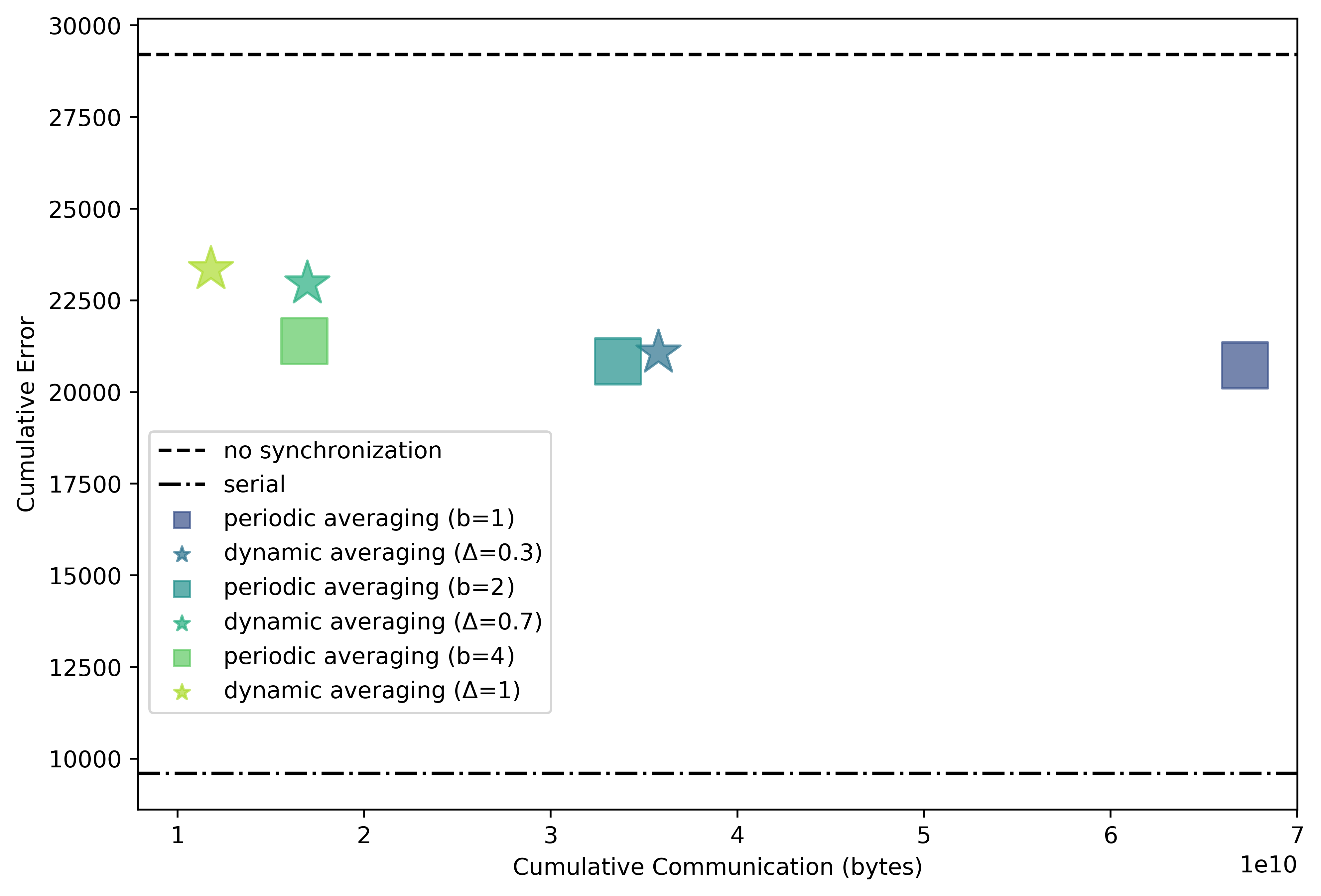}
		\label{fig:performance_adam}
}
\subfigure[$\uprule^{Rmsprop}$]{
		\includegraphics[width=5.5cm]{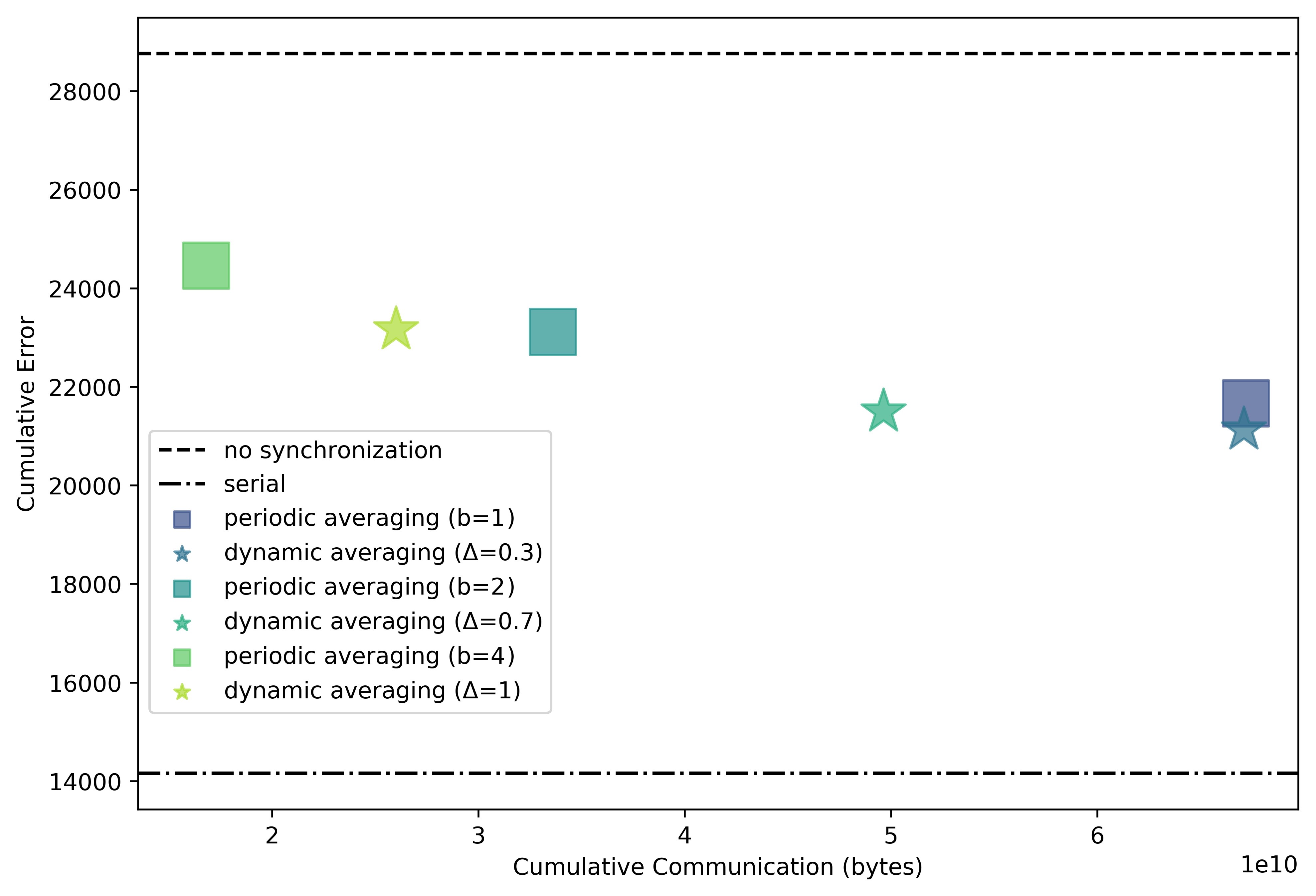}
		\label{fig:performance_rmsprop}
}
\caption{The averaged loss and the cumulative communication of the synchronization protocols for $m=10$ learners. Training is performed on MNIST for $2$ epochs.}
\label{fig:blackbox}
\end{figure}

\subsection{Scale-out Experiments}
\label{app:ssec:scale_out}
Communication size grows when the amount of the local learners is becoming larger, while it is still bounded for periodic and dynamic synchronization protocols (Section~\ref{sec:theory}). In order to see the behavior of periodic and dynamic synchronization while changing the number of learners, we executed an experiment with setups $m=10$, $m=100$ and $m=200$. The same synchronization operators were used in all three setups and the same number of examples was presented to each of the learners.

When the amount of examples per learner is fixed in order to have comparable by 
\begin{wrapfigure}{r}{0.6\textwidth}
	\includegraphics[width=0.6\textwidth]{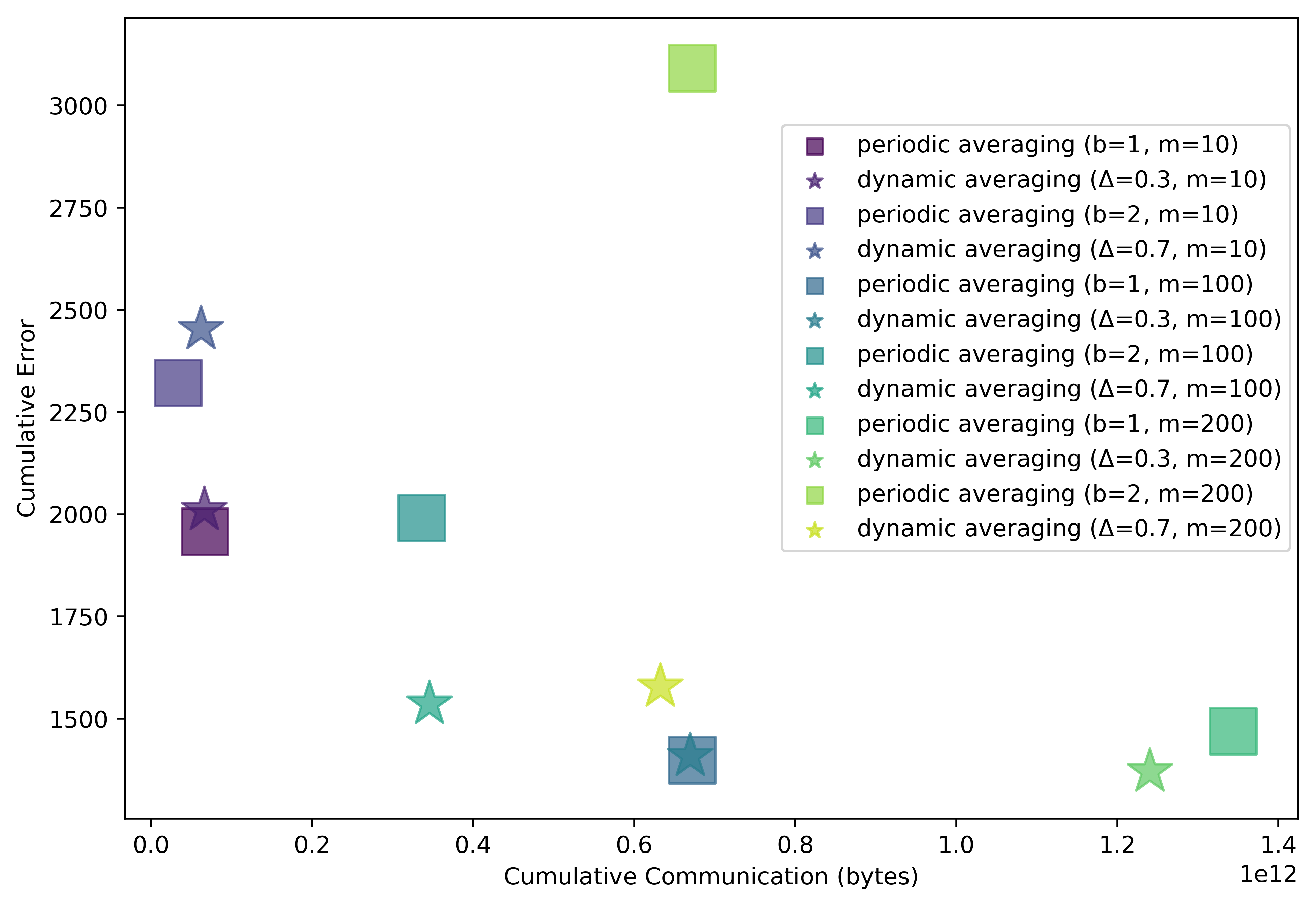}
	\centering
	\caption{
		The cumulative loss and the cumulative communication of the same synchronization protocols for a different amount of learners. Training is performed on MNIST for $2$, $20$ and $40$ epochs for $m=10$, $m=100$, $m=200$ setups correspondingly.}
	\label{fig:10_100_200_performance_ap}
\end{wrapfigure}
quality learners, the larger number of synchronizing models leads to a larger training dataset. 
As a consequence, it leads to a better performance of all of the setups. When the learners are saturated due to the long training and do not differ significantly enough to trigger the local conditions (see Section~\ref{sec:dynProtocol}), the advantage of the dynamic protocols over the periodic ones becomes more pronounced. 

The plot depicted in Figure~\ref{fig:10_100_200_performance_ap} shows the performance of two dynamic and two periodic protocols for the three scaling setups.
In order to make the cumulative loss comparable it was divided by the number of learners, i.e.,
the cumulative loss of $m=100$ learners is the sum of cumulative losses of all of them that is $10$ times 
more than the sum for $m=10$ learners---thus the first sum divided by $100$ is comparable to the second one divided by $10$. 
The plot shows that in the setup with $m=10$ learners $\syncop_{\Delta=0.7}$ shows a comparable result to $\syncop_{b=2}$ with the same amount of communication. 
At the same time with $m=100$ learners it already reaches much smaller cumulative loss. With $m=200$ learners $\syncop_{\Delta=0.3}$ requires less communication than $\syncop_{b=1}$ that was not the case for the previous setups.  that empirically shows advantage of dynamic synchronization for a large dataset setup.

The general approach of using local conditions to communicate efficiently has been successfully applied to massively distributed data streams~\citep{verner2011processing, sagy2010distributed, giatrakos2012prediction}. In the worst case, though, local conditions are violated in every round. Then dynamic averaging communicates as much as periodic averaging, but even then it scales at least as well as current decentralized learning approaches~\citep{mcmahan2017communication,jiang2017collaborative}. E.g., FedAvg~\citep{mcmahan2017communication} performs periodic averaging on a (random) fraction $C\in (0,1]$ of the $\totalLearners$ learners. Therefore, in the worst case dynamic averaging on $\totalLearners$ learners requires as much communication as FedAvg on $C\totalLearners$ learners. In practice, it requires substantially less communication even on the same number of learners. 



\subsection{Stability of the Dynamic Averaging Protocol Regarding Model Initializations}
\label{app:ssec:init}
\begin{figure}[ht]
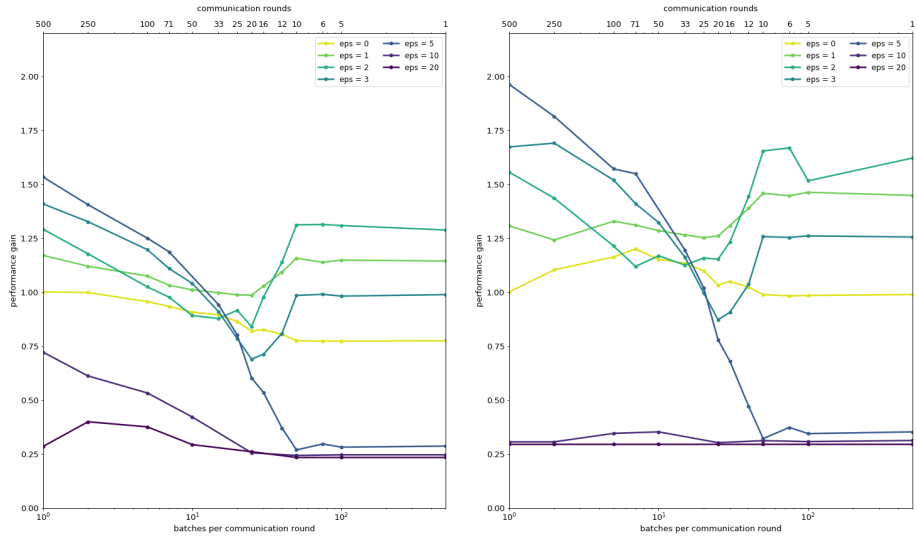

	\centering
	\subfigure{\label{fig:init_1_app}\includegraphics[width=6cm]{images/MNIST_CNN_initialization/initialization_experiment_static.png}}
	\hfill
	\subfigure{\label{fig:init_2_app}\includegraphics[width=6cm]{images/MNIST_CNN_initialization/initialization_experiment_dynamic.png}}
	\caption{Relative performances of averaged models on MNIST obtained from experiments with various heterogeneous model initializations parameterized by $\epsilon$ and various numbers of $b$. All averaged model performances are compared to an experiment with homogeneous model initializations ($\epsilon = 0$) and $b = 1$. \subref{fig:init_1_app} shows results for periodic averaging, \subref{fig:init_2_app} for dynamic averaging.}
	\label{fig:init_ap}
\end{figure}
On MNIST, naive averaging of fully trained models fails if these models are initialized differently (cf. \cite{mcmahan2017communication}). We study the transition from homogeneously initialized and converging learning configurations to heterogeneously initialized and failing models to better understand how the regularizing effects of model averaging impact the quality of the averaged model.
Failing model averaging is due to distances between the weight vectors of the different learners which are large compared to the scale of local convexity around local minima. In decentralized deep learning, there are several potential sources for such large distances. Two important ones are large numbers of local mini-batches, $b$, within one communication round and heterogeneous initializations. We leave out the learning rate in this experiment, though important as well, and leave the study of its impact for future work.

To parameterize the transition from homogeneous to heterogeneous initializations, we start with a homogeneous initialization according to \citet{Glorot10understandingthe} and impose noise at different scales $\epsilon$ on the homogeneous initialization. The noise scale $\epsilon$ is measured relative to the scale of the homogeneous initialization.
We conduct experiments with $m = 10$ learners, $B = 10$ and $500$ training examples per learner for a grid of $\epsilon$ and $b$ combinations. In Fig. \ref{fig:init_ap}, each point corresponds to the average model performance after running one such experiment. The $b$ dependency of the averaged model performance is shown on the abscissa, noise scale $\epsilon$ serves as curve parameter. Fig. \ref{fig:init_ap}\subref{fig:init_1_app} shows the results for periodic averaging, Fig. \ref{fig:init_ap} \subref{fig:init_1_app} for dynamic model averaging. The averaged model accuracies are not shown as absolute values but relative to the configuration with $\epsilon = 0$ and $b = 1$ which corresponds to homogeneously initialized models which communicate after processing one mini-batch.

For homogeneously initialized models, i.e., $\epsilon = 0$, we find a weak dependence of resulting model performance on the number of local mini-batches. Even configurations with very large numbers of local batches between two subsequent model averagings lead to convergence (see \cite{mcmahan2017communication}). However, this finding can be extended to the heterogeneous case, if the scales of these heterogeneities are at the scale of the underlying homogeneous initialization, e.g., $\epsilon \in \{1,2,3\}$. For large heterogeneities, however, model averaging fails, e.g. for $\epsilon = 20$. The transition between these two regimes occurs between $\epsilon = 5$ and $\epsilon = 10$, which show a strong dependency of model convergence on the number of local mini-batches. This critical scale of heterogeneity imposes a constrain on the choice of the dynamic protocol parameter $\Delta$ which indirectly determines the average distances between the different weight vectors before model averaging.

\section{Details on Theory}
\label{app:sec:theory}
In the following we prove that for the base learning algorithm stochastic gradient descent (SGD), dynamic averaging for deep neural networks retains the regret bound of periodic averaging. Before proving this result, we introduce the notion of regret.
The cumulative loss with respect to a reference model (e.g., the best model in hindsight) is denoted \defemph{regret} 
\[
\regret{}{\totalRounds,\totalLearners} = \sum_{\round=1}^\totalRounds\sum_{\learner=1}^{\totalLearners} \loss_{\round}^{\learner}(\model_{\round}^{\learner}) - \loss_{\round}^{\learner}(\model^*)\enspace .
\]
Guarantees are given by a \defemph{regret bound}, i.e., for all reference models $\model^*\in\modelSpace$ and all sequences of losses it holds that the regret 
is smaller than the regret bound $\regretbound{}{\totalRounds,\totalLearners}$ (e.g, for serial Stochastic Gradient Descent $\uprule^{SGD}$ with convex loss functions the regret is bounded by $\mathcal{O}(\sqrt{\totalRounds})$, for a non-synchronizing decentralized learning system using $\totalLearners\in\N$ learners and $\uprule^{SGD}$, the regret is bounded by $\mathcal{O}(\sqrt{\totalLearners\totalRounds})$~\citep{nesterov2013introductory}).

We now show that, given two model configurations $\mathbf{d}$ and $\mathbf{s}$, applying the dynamic averaging operator $\dynavg$ to $\mathbf{d}$ and the static averaging operator $\syncop_b$ to $\mathbf{s}$ increases their average squared pairwise model distances by at most $\Delta$.
\begin{lm}[\citet{kamp2016kernels}]
Let $\mathbf{d}_\round,\mathbf{s}_\round\in\modelSpace^\totalLearners$ be model configurations at time $\round \in \N$. Then
\[
\frac{1}{\totalLearners}\sum_{\learner=1}^\totalLearners\|\dynavg(\mathbf{d}_t)^\learner-\syncop_b(\mathbf{s}_\round)^\learner\|^2\leq\frac{1}{\totalLearners}\sum_{\learner=1}^\totalLearners\|d_\round^\learner-s_\round^\learner\|^2+\Delta\enspace .
\]
\label{syncLemma}
\end{lm}
\begin{proof}
We consider the case $\round \mod b = 0$ (otherwise the claim follows immediately).
Expressing the pairwise squared distances via the difference to $\overline{d_t}$ and using the definitions of $\syncop_b$ and $\dynavg$ we can bound
\begin{equation*}
\begin{split}
&\frac{1}{\totalLearners}\sum_{\learner=1}^\totalLearners\|\dynavg(\mathbf{d}_t)^\learner-\syncop_b(\mathbf{s}_t)^\learner\|^2=\frac{1}{\totalLearners}\sum_{\learner=1}^\totalLearners\|\dynavg(\mathbf{d}_t)^\learner-\overline{d_t}+\overline{d_t}-\overline{s_t}\|^2\\
=& \underbrace{\frac{1}{\totalLearners}\sum_{\learner=1}^\totalLearners \|\dynavg(\mathbf{d}_t)^\learner - \overline{d_t} \|^2}_{\leq \Delta \text{, by (ii) of Def.~\ref{def:dynamicsync}}} + 2\innerprod{ 
\underbrace{\frac{1}{\totalLearners}\sum_{\learner=1}^\totalLearners \dynavg(\mathbf{d}_t)^\learner - \overline{d_t}}_{=0 \text{, by (i) of Def.~\ref{def:dynamicsync}}}}{\overline{d_t}-\overline{s_t}} + \|\overline{d_t}-\overline{s_t}\|^2 \\
\leq&\Delta+\|\frac{1}{\totalLearners}\sum_{\learner=1}^\totalLearners(d_\round^\learner-s_\round^\learner)\|^2 = \Delta+\frac{1}{\totalLearners}\sum_{\learner=1}^\totalLearners\|d_\round^\learner-s_\round^\learner\|^2 \enspace .
\end{split}
\end{equation*}  
\qed
\end{proof}
Using this, we provide the proof of Theorem~2 in~\citet{boley2013communication}. For that, we first recapitulate the Theorem.
\begin{thm}[\citet{boley2013communication}]
\label{thm:regret}
Let $\loss$ be an $L$-Lipschitz loss function 
and the update rule $\uprule$ be a contraction with constant $c\in\R$.
Then, for batch sizes $b \geq \log^{-1}_2 c^{-1}$ and divergence thresholds $\Delta \leq \sfrac{\epsilon}{2L}$, the average regret of using a partial synchronization operator $\dynavg$ instead of $\syncop_b$ is bounded by $\epsilon$, i.e., for all rounds $\round \in \N$ it holds that the average regret
$
1/\totalLearners\sum_{\learner=1}^\totalLearners |\loss_\round^\learner(d_\round^\learner)-\loss_\round^\learner(s_\round^\learner)|
$
is bounded by $\epsilon$ where $d$ and $s$ denote the models at learner $\learner$ and time $\round$ maintained by $\dynavg$ and $\syncop_b$, respectively.
\end{thm}
\begin{proof}
We proof the claim within two steps. First we note that a regret bound is induced by a bound on the average distances between pairs of models at the local nodes. Then we show that such a bound is retained between the local model pairs resulting from static and dynamic synchronization.
Using the Lipschitz continuity of $\loss$ we can see that the average regret at round $\round$ is bounded as 
$\abs{\loss_\round^\learner(d_\round^\learner)-\loss_\round^\learner(s_\round^\learner)} \leq L\|d_\round^\learner-s_\round^\learner\|$.
Hence, for the desired average regret bound of $\epsilon$ it is sufficient to show that at all times $\round \in \N$ it holds that the average pair-wise model distance at the local learners is bounded by $2\Delta=\epsilon/(L)$, i.e.,
\begin{equation}
\label{eq:pairwisecond}
\frac{1}{\totalLearners}\sum_{l=1}^k \|d_\round^\learner-s_\round^\learner\|\leq 2\Delta \enspace .
\end{equation}
We now show that Eq.~\eqref{eq:pairwisecond} is retained throughout all rounds $t \in \N$. Before the first synchronization, i.e., for $t \leq b$, both weight sequences are identical and the bound holds. Moreover, if $\round-1$ is not a synchronization step, i.e., $\round-1 \mod b \neq 0$, the bound is preserved for $\mathbf{d}_{\round}$ and $\mathbf{s}_{\round}$ due to $\uprule$ being a contraction.
Hence, the crucial case is $\round > b$ with $(\round-1) \mod b = 0$. Using Lemma~\ref{syncLemma} we get that
\[
\frac{1}{\totalLearners}\sum_{\learner=1}^\totalLearners\|\dynavg(\mathbf{d}_t)^\learner-\syncop_b(\mathbf{s}_\round)^\learner\|^2\leq\underbrace{\frac{1}{\totalLearners}\sum_{\learner=1}^\totalLearners\|d_\round^\learner-s_\round^\learner\|^2}_{(\ast)}+\Delta\enspace .
\]
Applying the contraction property of $\uprule$ on $(\ast)$ yields
\[
\frac{1}{\totalLearners}\sum_{\learner=1}^\totalLearners\|d_\round^\learner-s_\round^\learner\|^2 \leq c^b \underbrace{\frac{1}{\totalLearners}\sum_{\learner=1}^\totalLearners\|d_{\round-b}^\learner-s_{\round}^\learner\|^2}_{\leq \Delta\text{ by IH}}\leq c^b\Delta+\Delta\leq 2\Delta\enspace .
\]
The last inequality follows from the fact that $b\geq \log^{-1}_2 c^{-1} = \log_c\sfrac{1}{2}$.
\qed
\end{proof}
It follows that if the assumptions of Theorem~\ref{thm:regret} hold, then dynamic averaging retains the regret bound of periodic averaging.

It remains to show that SGD is a contraction. For convex loss functions, one can show~\citep{zinkevich/nips/2010} that it is a contraction for sufficiently small constant learning rates: for $\eta \leq (\rho L+\lambda)^{-1}$ the updates do contract with constant $c=1-\eta\lambda$. Here, $\eta,\lambda\in\R_+$ denote the learning rate and regularization parameter, $L\in\R_+$ the Lipschitz constant of the loss function, and $\rho\in\R_+$ the data radius (i.e., for all $\sample\in\sampleSpace\|\sample\|_2\leq \rho$). 

This also holds for the non-convex case if for any round $\round\in\totalRounds$ the loss function is locally convex in a bounded region around the span of the models in $\mathbf{d}_\round$ and $\mathbf{s}_\round$. Since the distance of all models in $\mathbf{d}_\round$ to the models $\mathbf{s}_\round$ is bounded by $2\Delta$ and the distance of all models in $\mathbf{d}_\round$ to their average $\overline{d_\round}$ is bounded by $\Delta$, all models lie in a $2\Delta$ bounded region around $\overline{d_\round}$. If moreover the update magnitude is bounded by $R$, then all models remain in a $2\Delta R$ bounded region. Since the update magnitude for SGD is given by the norm of the gradient and the learning rate, the update magnitude can be bounded by $R=\eta L$. Thus for all models, if in any round $\round\in\N$ the loss function is locally convex in a $2\Delta R$-radius around $\overline{d_\round}$, then SGD is a contraction.
To see that this is a non-trivial but realistic assumption, see~\citet{sanghavi2017local, Wang2017StochasticNO} on local convexity and~\citet{nguyen2017loss,keskar2017on} on the the loss surface of deep learning.
It follows that under these assumptions, dynamic averaging with SGD retains the regret bound of periodic averaging. 
\begin{crl}
	Let $\loss$ be an $L$-Lipschitz loss function, $\lambda\in\R$ a regularization parameter, $\eta \leq (\rho L+\lambda)^{-1}$ a learning rate, $\rho\in\R_+$ the data radius, $b \geq \log^{-1}_2 c^{-1}$ the batch sizes and $\Delta \leq \sfrac{\epsilon}{2L}$ the divergence threshold, dynamic averaging retains the regret bound of periodic averaging. In round $\round\in\N$, let $\mathbf{d}_\round$ and $\mathbf{s}_\round$ denote the models maintained by dynamic and periodic averaging, respectively. If in any round $\round\in\N$ the loss function is locally convex in a $2\Delta R$-radius around $\overline{d_\round}$, then dynamic averaging with stochastic gradient descent retains the regret bound of periodic averaging.  
	%
	\label{app:cor:lossbounddynamic}
\end{crl}
Together with Proposition~\ref{prop:contAvgEqualMiniBatch} it follows that dynamic averaging is consistent as in Definition~\ref{def:efficiency}.

\pagebreak
\section{Federated Learning: Different Sampling Rates and non-iid Data}
In order to analyze the proposed approach theoretically, we used a strict notion with fixed, balanced sampling rate and iid data. In contrast,~\citet{mcmahan2017communication} introduced federated learning as a learning task with (i) non-iid data, (ii) unbalanced data sampling rates, (iii) massively distributed systems, and (iv) limited communication infrastructure. 
\begin{algorithm2e}[h]
	\caption{Dynamic Averaging Protocol for Unbalanced Data}
	\label{alg:protocolUnbalanced}
	\smallskip
	\textbf{Input:}    divergence threshold $\Delta$, batch size $b$\\    
	\smallskip
	\textbf{Initialization:}\\
	\begin{algorithmic}[0]
		\STATE local models $\model^1_1,\dots,\model^\totalLearners_1 \leftarrow$ one random $\model$
		\STATE reference vector $r \leftarrow \model$
		\STATE violation counter $v \leftarrow 0$
	\end{algorithmic}
	\smallskip
	\textbf{Round }$\round$\textbf{ at node }$\learner$\textbf{:}\\
	\begin{algorithmic}[0]
		\STATE \textbf{observe} $\locsample_\round^\learner\subset\sampleSpace\times\outputSpace$ with $\left|\locsample_\round^\learner\right|=B^\learner$ 
		\STATE \textbf{update} $\model^\learner_{\round-1}$ using the learning algorithm $\uprule$\\
		\IF{$\round \mod b=0$ \textbf{and} $\|\model^\learner_{\round}-r\|^2> \Delta$}
		\STATE \textbf{send} $\model^\learner_{\round}$ and $B^\learner$ to coordinator (violation)
		\ENDIF
	\end{algorithmic}
	\smallskip
	\textbf{At coordinator on violation:}\\
	\begin{algorithmic}[0]
		\STATE \textbf{let} $\balancingSet$ be the set of nodes with violation
		\STATE $v\leftarrow v+\card{\balancingSet}$
		\STATE \textbf{if} $v=\totalLearners$ \textbf{then} $\balancingSet\leftarrow [\totalLearners]$, $v\leftarrow 0$
		\STATE $N\leftarrow \sum_{\learner\in \balancingSet}B^\learner$
		\WHILE{$\balancingSet \neq [\totalLearners]$ \textbf{and} $\left\|\frac{1}{N}\sum_{\learner \in \balancingSet}B^\learner\model^\learner_\round-r\right\|^2> \Delta$}
		\STATE \textbf{augment} $\balancingSet$ by augmentation strategy
		\STATE \textbf{receive} models from nodes added to $\balancingSet$
		\STATE $N\leftarrow \sum_{\learner\in \balancingSet}B^\learner$
		\ENDWHILE
		\STATE \textbf{send} model $\avgmodel=\frac{1}{N}\sum_{\learner \in \balancingSet}B^\learner\model^\learner_\round$ to nodes in $\balancingSet$
		\STATE \textbf{if} $\balancingSet=[\totalLearners]$ also set new reference vector $r\leftarrow\avgmodel$
	\end{algorithmic}
\end{algorithm2e}\\

Naturally, dynamic averaging is well suited for massively distributed systems with limited communication infrastructure---in the worst case it performs similar to periodic averaging, in the best case it has similar predictive performance with orders of magnitude less communication. Regarding non-iid data,~\citet{mcmahan2017communication} provide empirical indications that differences in local data distributions do not significantly deteriorate the learning process. While this should hold for dynamic averaging as well, further research is required to substantiate this claim, both for federated learning and dynamic averaging (cf.~\citet{smith2017federated}).

In order to tackle the problem of unbalanced sampling rates, a simple modification of the dynamic averaging protocol is required. For that, assume that each local learner $\learner\in [\totalLearners]$ observes $B^\learner\in\N$ samples. 
We can account for the different sampling rates in the training process by using a weighted model averaging, where the weight of each model depends on the number of observed examples. Let $N=\sum_{\learner=1}^\totalLearners B^\learner$ be the total number of samples observed in each round. Then, the weighted average of models is given by
\[
\avgmodel=\frac{1}{N}\sum_{\learner=1}^\totalLearners B^\learner\model^\learner_\round\enspace .
\]
Note that this can be generalized analogously to time-dependent sampling rates $B^\learner_\round$. 
Using this weighted average, the dynamic averaging protocol for unbalanced data is given in Algorithm~\ref{alg:protocolUnbalanced}.

\end{document}